\documentclass[]{style}
\usepackage{float}
\usepackage{graphicx}
\usepackage{titletoc}

\newcommand{\Phisst}{\Phi_{s,t}}
\newcommand{\Xst}{X_{s,t}}
\newcommand{\Estt}{\mathcal{E}_{s,t}}
\newcommand{\psist}{\psi_{s,t}}

\newcommand{\xseps}{\boldsymbol{x}_{s}^{\eps}}
\newcommand{\peps}{p_{\eps| s,t}}
\newcommand{\bx}{\boldsymbol{x}}
\newcommand{\by}{\boldsymbol{y}}
\newcommand{\bz}{\boldsymbol{z}}
\newcommand{\be}{\boldsymbol{e}}
\newcommand{\eps}{\epsilon}
\newcommand{\beps}{\boldsymbol{\epsilon}}

\newcommand{\E}{\mathcal{E}}
\newcommand{\Real}{\mathbb{R}}

\usepackage[english]{babel}
\usepackage{xcolor}
\usepackage{color}
\usepackage{wrapfig}
\usepackage[table]{xcolor}

\usepackage{amsmath}
\usepackage{amssymb}
\usepackage{amsthm}
\usepackage{graphicx}
\usepackage{booktabs}      
\usepackage{array}
\usepackage{multirow}      
\usepackage{caption}  
\usepackage{amsmath}
\usepackage[table]{xcolor}  
\usepackage{xcolor}
\usepackage{microtype}
\usepackage{parskip}

\definecolor{bg}{HTML}{E7E7FE}
\definecolor{accent}{HTML}{6B67EE}
\definecolor{accent-bright}{HTML}{6B67EE}
\definecolor{mist}{HTML}{F0F0F8}
\definecolor{panelborder}{HTML}{9998B3}
\definecolor{muted}{HTML}{4a4a6a}

\newtcolorbox{figurepanel}{
    enhanced,
    colback=mist,
    colframe=panelborder,
    boxrule=0.8pt,
    arc=6pt,
    left=10pt,right=10pt,top=10pt,bottom=10pt
}

\newtcolorbox{samplebox}{
    enhanced,
    colback=accent!4!bg!60!white,
    colframe=accent!55!black,
    boxrule=0.6pt,
    arc=4pt,
    left=8pt,right=8pt,top=7pt,bottom=7pt
}
\definecolor{mybg}{HTML}{EAE9FF} 
\definecolor{outline}{HTML}{9793F8} 
\newtcolorbox{borderbox}[1][]{
    colback=mybg,
    colframe=outline,
    top=2pt, left=2pt, right=2pt, bottom=2pt,
    #1
}

\newcommand{\nfeheader}[1]{%
    {\sffamily\bfseries\large #1}
}
\newcommand{\methodtitle}[1]{%
    {\bfseries #1}
}
\newcommand{\metrics}[2]{%
    {\color{muted} Gen.\ PPL: \textbf{#1} \\
    \hfill Entropy: \textbf{#2}}
}

\usepackage{algorithm}
\usepackage{algpseudocode}

\usepackage[noEnd=false]{algpseudocodex}

\include{math_commands}
\usepackage{thmtools}
\usepackage{thm-restate}

\newtcolorbox{propbox}{
    enhanced,
    colback=accent!4!bg!60!white,
    colframe=accent, 
    boxrule=0pt,
    leftrule=0pt,      
    arc=0pt,
    left=8pt, right=8pt, top=8pt, bottom=8pt,
    fonttitle=\bfseries\sffamily,
    coltitle=accent-bright
}
\tcolorboxenvironment{proposition}{enhanced,
    colback=accent!4!bg!60!white,colframe=accent,boxrule=0pt,leftrule=0pt,arc=0pt,
    left=8pt, right=8pt, top=8pt,bottom=8pt,fonttitle=\bfseries\sffamily,coltitle=accent-bright
}
\tcolorboxenvironment{remark}{enhanced,
    colback=accent!4!bg!60!white,colframe=accent,boxrule=0pt,leftrule=0pt,arc=0pt,
    left=8pt, right=8pt, top=8pt,bottom=8pt,fonttitle=\bfseries\sffamily,coltitle=accent-bright
}
\tcolorboxenvironment{definition}{enhanced,
    colback=accent!4!bg!60!white,colframe=accent,boxrule=0pt,leftrule=0pt,arc=0pt,
    left=8pt, right=8pt, top=8pt,bottom=8pt,fonttitle=\bfseries\sffamily,coltitle=accent-bright
}
\tcolorboxenvironment{box*}{enhanced,
    colback=accent!4!bg!60!white,colframe=accent,boxrule=0pt,leftrule=0pt,arc=0pt,
    left=8pt, right=8pt, top=8pt,bottom=8pt,fonttitle=\bfseries\sffamily,coltitle=accent-bright
}
\tcolorboxenvironment{lemma}{enhanced,
    colback=accent!4!bg!60!white,colframe=accent,boxrule=0pt,leftrule=0pt,arc=0pt,
    left=8pt, right=8pt, top=8pt,bottom=8pt,fonttitle=\bfseries\sffamily,coltitle=accent-bright
}
\tcolorboxenvironment{restatable}{enhanced,
    colback=accent!4!bg!60!white,colframe=accent,boxrule=0pt,leftrule=0pt,arc=0pt,
    left=8pt, right=8pt, top=8pt,bottom=8pt,fonttitle=\bfseries\sffamily,coltitle=accent-bright
}
\tcolorboxenvironment{theorem}{enhanced,
    colback=accent!4!bg!60!white,colframe=accent,boxrule=0pt,leftrule=0pt,arc=0pt,
    left=8pt, right=8pt, top=8pt,bottom=8pt,fonttitle=\bfseries\sffamily,coltitle=accent-bright
}
\tcolorboxenvironment{corollary}{enhanced,
    colback=accent!4!bg!60!white,colframe=accent,boxrule=0pt,leftrule=0pt,arc=0pt,
    left=8pt, right=8pt, top=8pt,bottom=8pt,fonttitle=\bfseries\sffamily,coltitle=accent-bright
}

\declaretheorem[
  name=Definition,
  numberwithin=section,
  refname={Definition,Definition},
  Refname={Definition,Definition}
]{definition}

\declaretheorem[]{box*}

\declaretheorem[
  name=Remark,
  numberwithin=section,
  refname={Remark,Remarks},
  Refname={Remark,Remarks}
]{remark}
\declaretheorem[
  name=Assumption,
  numberwithin=section,
  refname={Assumption,Assumptions},
  Refname={Assumption,Assumptions}
]{assumption}

\title{\fontsize{0.6cm}{0.72cm}\selectfont Expanding Flow Maps}

\author[1]{Sophia Tang}
\author[1,2]{\, Pranam Chatterjee}

\affiliation[1]{Department of Computer and Information Science, University of Pennsylvania}
\affiliation[2]{Department of Bioengineering, University of Pennsylvania}

\abstract{Flow-based generative models have enabled remarkable progress in fast and controllable generation across continuous and discrete state spaces, yet existing parameterizations are constrained to fixed dimensions or fixed sequence lengths. Here, we introduce \textbf{Expanding Generative Flows} (EFlows), which define flows between distributions of \textit{increasing dimensionality} along an expanding interpolant that grows the state by augmenting it with conditional noise. Building on this construction, we propose \textbf{Expanding Flow Maps} (EFMs), a new class of flow maps that distill the expanding interpolant into efficient few-step generative models. Each EFM factors the map between any two timesteps into two learnable operations: an \textit{expand operator}, which augments the state space with new coordinates or tokens conditioned on the current state, and a \textit{transport map}, which pushes the expanded state forward along the interpolant. Composing these operators yields a single map that jointly \textit{expands} and \textit{denoises} the state, recovering existing fixed-canvas flows and flow maps as the special case in which the expand operator is the identity. We further extend the framework to the discrete simplex, enabling variable-size graph generation and variable-length sequence generation. Across both continuous and discrete modalities, we establish EFlows and EFMs as a principled framework for settings in which output size is itself a learned, controllable degree of freedom.

\vspace{0.5em}%
\textbf{\sffamily\bfseries Correspondance: }\href{sophtang@engineering.upenn.edu}{\texttt{sophtang@engineering.upenn.edu}}, \href{pranam@upenn.edu}{\texttt{pranam@upenn.edu}}}
\begin{document}
\maketitle

\renewcommand{\footnoterule}{%
  \kern -3pt
  \hrule width \linewidth
  \kern 2.6pt
}
\vspace{-0.2cm}


\section{Introduction}
Flow maps have recently emerged as a unifying framework for few-step generative modeling, supporting both continuous \citep{boffi2025build} and discrete \citep{roos2026categorical, lee2026flow, potaptchik2026discrete} state spaces. Building on the success of score \citep{song2020score, ho2020denoising} and flow-based \citep{lipman2022flow, tong2023improving, albergo2025stochastic} models, flow maps cast generation as a continuous-time interpolant between noise and data and learn to parameterize discrete jumps along the ODE trajectory between any pair of timesteps \citep{boffi2025build, geng2025mean, song2023consistency, kim2023consistency}. They are typically trained either by \textit{distillation from a teacher flow} \citep{boffi2024flow, frans2024one, zhou2025inductive, salimans2024multistep} or by \textit{self-distillation from data} \citep{boffi2025build}, enabling sampling in as few as one or two function evaluations while retaining the modeling flexibility of their multi-step counterparts.

Despite this progress, existing flow maps are fundamentally restricted to generation on a \textit{fixed canvas}: a continuous state space of fixed dimensionality $d$, or a discrete state space of fixed sequence length $L$. This rigid constraint conflicts with the structure of many real-world generative tasks, where the size of the output is itself a quantity to be modeled -- for example, variable-resolution images and 3D shapes, audio and video of arbitrary duration, language with unknown sequence length, and multi-modal data that interleaves modalities of differing intrinsic dimensionality. Addressing these settings requires a generative framework whose state space can \textit{grow} during inference, rather than being committed to at initialization. This raises the central question motivating this work:

\begin{center}
    \textit{How can we enable few-step generation of both continuous and discrete data with adaptive dimensionality at inference?}
\end{center}

We address this question by introducing \textbf{Expanding Flow Maps} (EFMs), a general framework for few-step generative modeling over both continuous and discrete data with variable dimensionality and sequence length. The core idea is to factor the flow map between any two timesteps into two learnable operations: an \textit{expand operator}, which expands the state space by inserting new coordinates or tokens, and a \textit{transport map}, which moves the expanded state forward along the interpolant. Composing these operators yields a single map that jointly \textit{expands} and \textit{denoises} the state, positioning standard fixed-dimensional flow maps \citep{boffi2025build, roos2026categorical, lee2026flow, potaptchik2026discrete} as the special case in which the expand operator is the identity.

Our \textbf{main contributions} can be summarized in the following points:
\begin{enumerate}
    \item \textbf{Expanding Generative Flows:} We introduce \textit{Expanding Generative Flows} (EFlows), a class of generative models defined on an \textit{expanding interpolant} between distributions of increasing dimensionality. We formulate it as a piecewise-deterministic Markov process consisting of a transport term that describes smooth denoising and a jump kernel that describes jumping to higher dimensionalities.
    \item \textbf{Expanding Flow Maps:} We derive the \textit{Expanding Flow Map} (EFM), which enables jumping between arbitrary points along the expanding interpolant by composing an \textit{expand operator}, which augments a state with conditional noise, and a \textit{transport map}, which pushes the augmented state forward over the time interval toward the target distribution.
    \item \textbf{Discrete Expanding Flows and Maps:} We extend expanding generative flows and flow maps to the discrete simplex space, formulating expansion as insertions along a sequence or graph and unlocking few-step variable-length discrete generation in regimes where standard flow-based methods do not apply.
\end{enumerate}

Through experiments on molecular conformer generation, discrete molecular graph generation, and language modeling, we establish EFlow and EFMs as a \textit{single, unified recipe} for variable-dimensional generation across domains and at any sampling budget. Further discussion of related works is provided in App \ref{app:related-works}.

\begin{figure*}
    \centering
    \includegraphics[width=\linewidth]{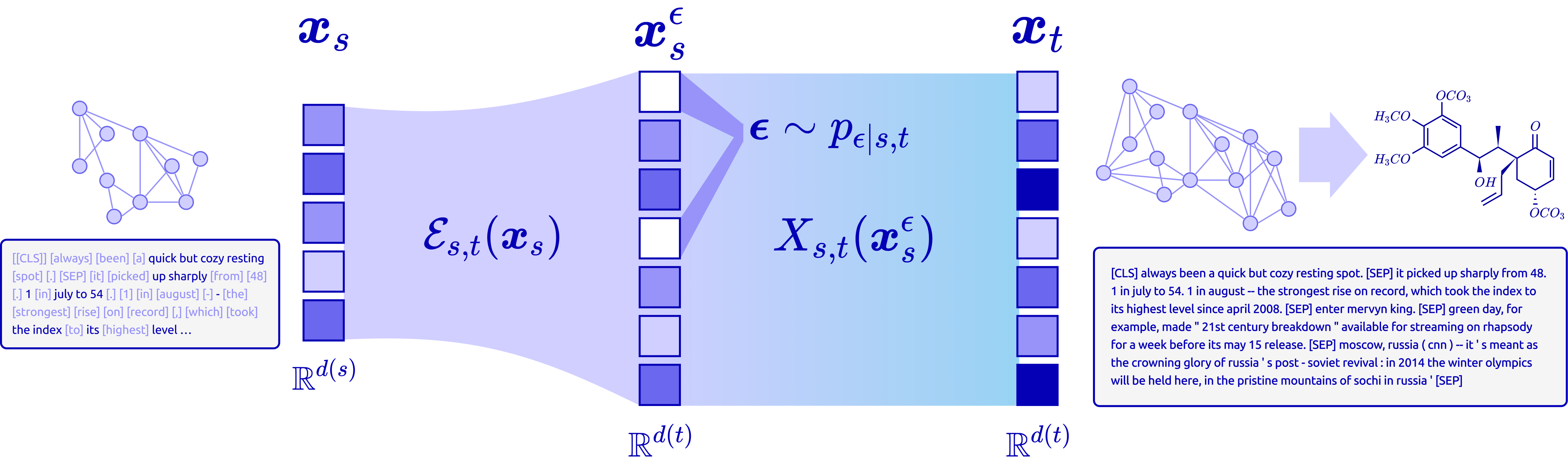}
    \caption{\textbf{Expanding Flow Maps.} Step from lower-dimensional space on $\mathbb{R}^{d(s)}$ at time $s$ to a higher dimensional space $\mathbb{R}^{d(t)}$ at time $t$ by \textbf{(i)} applying the expand operator $\mathcal{E}_{s,t}$ with augmented noise $\boldsymbol{\epsilon}\sim p_{\epsilon|s,t}$ and \textbf{(ii)} applying the transport map $X_{s,t}$ in the augmented space.}
    \label{fig:efm}
\end{figure*}

\section{Preliminaries}
\paragraph{Generative Flows and Stochastic Interpolants}
A unifying perspective on flow and diffusion models is provided by the \textit{stochastic interpolants} framework \citep{albergo2025stochastic, albergo2022building}, which constructs a continuous-time path between a prior $p_0$ and the data distribution $p_1 = p_{\text{data}}$. Given endpoints $\boldsymbol{x}_0 \sim p_0$ and $\boldsymbol{x}_1 \sim p_1$, an interpolant is a stochastic process:
\begin{align}
    \boldsymbol{x}_t = I_t(\boldsymbol{x}_0, \boldsymbol{x}_1),\quad \dot{\boldsymbol{x}}_t=b_t(\boldsymbol{x}_t), \quad t \in [0,1],
\end{align}
where $\boldsymbol{x}_{t=0}=\boldsymbol{x}_0\sim p_0$ and $\boldsymbol{x}_{t=1}=\boldsymbol{x}_1\sim p_1$ with intermediate marginals $\boldsymbol{x}_t\sim p_t=\text{Law}(I_t)$. The standard \textit{linear interpolant} is defined as $I_t(\boldsymbol{x}_0, \boldsymbol{x}_1)=(1-t)\boldsymbol{x}_0+t\boldsymbol{x}_1$. The velocity field $b_t(\boldsymbol{x})$ given a point $\boldsymbol{x}$ can be obtained as the conditional expectation of the time derivative of the interpolant at $\boldsymbol{x}_t=\boldsymbol{x}$:
\begin{align}
    b_t(\boldsymbol{x})=\mathbb{E}[\dot{\boldsymbol{x}}_t|\boldsymbol{x}_t=\boldsymbol{x}]
\end{align}
which transports samples from $p_0$ to $p_1$ along the ODE $d\boldsymbol{x}_t=b_t(\boldsymbol{x}_t)dt$ or SDE $d\boldsymbol{X}_t=b_t(\boldsymbol{X}_t)dt+\sigma_td\boldsymbol{B}_t$ with diffusion coefficient $\sigma_t$. To parameterize the velocity field given two empirical distributions, we minimize the \textit{conditional flow-matching objective} \citep{lipman2022flow, albergo2022building}:
\begin{align}
    \mathcal{L}_{\text{CFM}}(\hat{b}):=\int_0^1\mathbb{E}_{\bx_0,\bx_1}\left[\big\|\hat{b}_t(I_t(\boldsymbol{x}_0, \boldsymbol{x}_1))-\partial_tI_t(\boldsymbol{x}_0, \boldsymbol{x}_1)\big\|^2\right]dt\label{eq:cfm-loss}
\end{align}
where the expectation is over $(\boldsymbol{x}_0, \boldsymbol{x}_1)\sim p_0\otimes p_1$.

\paragraph{Flow Maps for Trajectory Distillation}
Numerically integrating the ODE or SDE above typically requires many small steps to accurately approximate the continuous-time dynamics, so a natural goal is to minimize the number of function evaluations needed to generate a target sample. This motivates consistency models \citep{song2023consistency, song2023improved} and flow maps \citep{boffi2025build, sabour2025align, geng2025mean, frans2024one, guo2025splitmeanflow}, which aim to distill the many-step integration process into a few-step process. Concretely, a flow map is a deterministic map $X_{s,t}:\mathbb{R}^d\to \mathbb{R}^d$ that directly outputs the state of the ODE at time $t$ given its state at an earlier time $s$:
\begin{align}
    \quad X_{s,t}(\boldsymbol{x}_s)=\boldsymbol{x}_t=\boldsymbol{x}_s+(t-s)v_{s,t}(\boldsymbol{x}_s), \quad \forall s,t \in [0,1]\label{eq:continuous-fm}
\end{align}
which can be defined as taking a step of size $(t-s)$ using the \textit{mean velocity} $v_{s,t}(\boldsymbol{x}_s)$ from $\boldsymbol{x}_s$. When $t=s$, the mean velocity reduces to the instantaneous velocity $v_{s,s}(\boldsymbol{x}_s)=b_s(\boldsymbol{x}_s)$ which can be trained with the standard CFM loss in (\ref{eq:cfm-loss}), referred to as the \textit{diagonal loss} $\mathcal{L}_{\text{diag}}$ as it trains the mean velocity along the diagonal where $t=s$. To parameterize $v_{s,t}(\boldsymbol{x}_s)$ for when $t\neq s$, we consider three equivalent conditions of the flow map $X_{s,t}$:
\begin{subequations}
\label{eq:consistency-constraint}
\begin{align}
    \textit{Lagrangian Condition:}\quad &\partial_tX_{s,t}(\boldsymbol{x}_s)=v_{t,t}(X_{s,t}(\boldsymbol{x}))\\
    \textit{Eulerian Condition:}\quad &\partial_sX_{s,t}(\boldsymbol{x})+v_{s,s}(\boldsymbol{x})\cdot\nabla_{\boldsymbol{x}_s} X_{s,t}(\boldsymbol{x})=0\\
    \textit{Semigroup Condition:}\quad &X_{u,t}(X_{s,u}(\boldsymbol{x}))=X_{s,t}(\boldsymbol{x})
\end{align}
\end{subequations}
for all $\boldsymbol{x}\in \mathbb{R}^d$ and $s,u,t\in [0,1]$. These conditions each translate into \textit{consistency objectives} $\mathcal{L}_{\text{cons}}$ by setting the condition equal to 0 and taking the squared residual with respect to the parameterized velocity $\hat{v}_{s,t}:\mathbb{R}^d\to \mathbb{R}^d$. Then, the full training objective becomes the sum of the \textit{diagonal} and \textit{consistency} objectives $\mathcal{L}_{\text{FM}}(\hat{v})=\mathcal{L}_{\text{diag}}(\hat{v})+\mathcal{L}_{\text{cons}}(\hat{v})$.

\paragraph{Flow Maps for Discrete Data}
Recently, flow maps have been extended to the discrete state space \citep{roos2026categorical, lee2026flow, potaptchik2026discrete}, where the states are defined as vectors on the simplex space $\Delta^{V-1}$ with vocabulary $|\mathcal{V}|=V$. A sequence of categorical distributions over the simplex is denoted $\boldsymbol{x}_t\in \mathcal{V}^L$. Representing discrete data as continuous distributions over the simplex enables the repurposing of standard flow-map objectives in the discrete setting. The equivalent parameterization of the mean velocity on the discrete state space is the \textit{mean denoiser} $\psi_{s,t}:\mathbb{R}^{L\times V}\to \mathbb{R}^{L\times V}$ which reduces to the standard denoiser on the diagonal $t=s$ and is defined by taking a full $(1-s)$ step in the direction of the mean velocity $v_{s,t}$ to the denoised state:
\begin{align}
    \psi_{s,t}(\boldsymbol{x})=\boldsymbol{x}+(1-s)v_{s,t}(\boldsymbol{x}), \quad v_{s,t}(\boldsymbol{x})=\frac{\psi_{s,t}(\boldsymbol{x})-\boldsymbol{x}}{1-s}\label{eq:mean-denoiser}
\end{align}
which yields the discrete analog of the flow map by substituting $v_{s,t}$ into (\ref{eq:continuous-fm}):
\begin{align}
    X_{s,t}(\boldsymbol{x})=\frac{1-t}{1-s}\boldsymbol{x}+\frac{t-s}{1-s}\psi_{s,t}(\boldsymbol{x}), \quad \forall s,t\in [0,1]
\end{align}
Crucially, the mean denoiser can be written as a weighted sum over the conditional expectation of the data distribution and lies on the simplex $\psi_{s,t}\in \Delta^{V-1}$, allowing it to be trained via the \textit{cross-entropy} analogs of the diagonal and consistency objectives:
\begin{align}
    \mathcal{L}_{\text{diag}}(\hat{\psi})=-\sum_{i=1}^LD_t(\boldsymbol{x}_t^i)\cdot\log \hat{\psi}_{t,t}(\boldsymbol{x}_t^i), \quad \mathcal{L}_{\text{cons}}(\hat{\psi})=-\sum_{i=1}^L\bar{\psi}_{s,t}\cdot\log \hat{\psi}_{s,t}(\boldsymbol{x}_s^i)
\end{align}
where $\bar{\psi}_{s,t}$ is defined via the discrete analogs of the \textit{lagrangian}, \textit{Eulerian}, and \textit{semigroup} conditions \citep{lee2026flow, potaptchik2026discrete}. 

\section{Expanding Generative Flows}
\label{sec:expanding-flows}
The \textbf{key limitation} of standard continuous flows is that the state space is fixed at initialization: $\boldsymbol{x}_0\in \mathbb{R}^d$ with fixed $d$ in the continuous state space and $\boldsymbol{x}_0\in \mathbb{R}^{L\times V}$ with fixed length $L$ in the discrete sequence space. We overcome this limitation by presenting a rigorous treatment of generative flows between distributions of \textit{different dimensionality}, specifically the setting where the source distribution has \textit{lower dimensionality} than the target distribution. We construct this in three pieces: an \textit{expanding interpolant} that grows the state along a dimension schedule, an \textit{expand operator} that lifts the source into the target space, and a \textit{piecewise deterministic Markov process (PDMP)} characterization that unifies expansion and transport into a single well-defined generative process.

\paragraph{Expanding Interpolant}
To accommodate generative tasks in which the size of the output is itself variable, we generalize the standard flow interpolant construction to allow the dimension to increase along the trajectory. Concretely, we replace the fixed dimension $d$ with a non-decreasing dimension schedule $d(t):[0,1]\to \mathbb{N}$ and let $\boldsymbol{x}_t\in \mathbb{R}^{d(t)}$, so that the marginal $p_t\in \mathcal{P}(\mathbb{R}^{d(t)})$ is defined on the state space whose dimensionality depends on $t$.

\begin{borderbox}
\begin{center}
\textbf{Generative Flow with Expanding Dimensionality}
\begin{align}
    \mathcal{P}(\mathbb{R}^{d(0)})\ni p_0\to \dots \to p_s\to\dots \to  p_t\to\dots \to  p_1\in \mathcal{P}(\mathbb{R}^{d(1)})
\end{align}
where $d(s)\leq d(t)$ for all $ 0\leq s < t\leq 1$.
\end{center}
\end{borderbox}

The central challenge of defining this type of flow is that, for $s<t$ with $d(s)<d(t)$, the source state $\boldsymbol{x}_s\in \mathbb{R}^{d(s)}$ and the target state $\boldsymbol{x}_t\in \mathbb{R}^{d(t)}$ live in spaces of different dimension. Consequently, there is no diffeomorphism $\mathbb{R}^{d(s)}\to \mathbb{R}^{d(t)}$ and no fixed-dimension velocity field that can transport $p_s$ to $p_t$. To resolve this, we \textit{augment} the source state with additional coordinates of noise drawn from a tractable conditional distribution $p_{\epsilon|s,t}\in \mathcal{P}(\mathbb{R}^{d(t)-d(s)})$, lifting the source to the target dimensionality before applying a transport map.

\begin{remark}[Expanding Interpolants Generalize Standard Generative Interpolants]
    The standard fixed-dimensional flow ODE is the special case of the expanding generative flow in which the dimension $\mathbb{R}^{d(t)}$ is constant over $t\in [0,1]$.
\end{remark}

\paragraph{Expand Operator}
To parameterize the expanding generative flow, we define a \textit{two-stage} process that, at each discrete time step, expands the state space with augmented noise coordinates and then integrates the interpolant from the augmented state to the target time step. The first stage is carried out by the \textit{expand operator} $\mathcal{E}_{s,t}$.
\begin{borderbox}
    \begin{center}
        \textbf{Expand Operator}
        \begin{align}
            \mathcal{E}_{s,t}(\boldsymbol{x}_s,\boldsymbol{\epsilon}): \mathbb{R}^{d(s)}\times\mathbb{R}^{d(t)-d(s)}\to\mathbb{R}^{d(t)}, \quad \boldsymbol{\epsilon}\sim p_{\epsilon|s,t}\in \mathbb{R}^{d(t)-d(s)}\label{eq:expand-operator}
        \end{align}
        which defines the placement scheme determining where the coordinates of $\boldsymbol{\epsilon}$ are inserted relative to $\bx_s$.
    \end{center}
\end{borderbox}

The augmented noise variable $\boldsymbol{\epsilon}$ adds new coordinates to the growing state space, which can then be mapped to a sample from the target distribution $p_t$. The expand operator $\mathcal{E}_{s,t}$ combines the existing coordinates of $\boldsymbol{x}_s$ with the latent coordinates of $\boldsymbol{\epsilon}$ according to a placement scheme specifying where in the augmented state the new coordinates appear. $\mathcal{E}_{s,t}$ can take several instantiations, including:
\begin{enumerate}
    \item [(i)] \textbf{Concatenation}: $\mathcal{E}_{s,t}(\boldsymbol{x}_s, \boldsymbol{\epsilon})=\text{Cat}(\boldsymbol{x}_s, \boldsymbol{\epsilon})$ appends the new coordinates to the end of the state. This is the natural instantiation for autoregressive sequence generation or for appending points to a point cloud.
    \item[(ii)] \textbf{Positional Insertion}: $\mathcal{E}_{s,t}(\boldsymbol{x}_s, \boldsymbol{\epsilon})$ interleaves the entries of $\boldsymbol{\epsilon}$ at a subset of positions within $\{1, \dots, d(t)\}$ in the augmented state. This instantiation is useful for any-length sequence generation or infilling.
    \item[(iii)] \textbf{Child Expansion}: $\mathcal{E}_{s,t}(\boldsymbol{x}_s, \boldsymbol{\epsilon})$ attaches each new coordinate to a parent entry $\mathrm{pa}(i)$ of $\boldsymbol{x}_s$ and initializes it relative to that parent, $\boldsymbol{x}^\epsilon_s[i]=\boldsymbol{x}_s[\mathrm{pa}(i)]+\sigma\boldsymbol{\epsilon}_i$. This instantiation is suitable for coarse-to-fine generation on structured states.
\end{enumerate}
The augmented noise and its placement may be deterministic (e.g., always concatenating i.i.d.\ Gaussian noise to the end) or predicted via a parameterized network $\hat{\mathcal{E}}$ (e.g., predicting the number of insertions at each gap). In all cases, the expand operator lifts the state to $\mathbb{R}^{d(t)}$ so that transport along the interpolant from $s$ to $t$ can proceed via a $d(t)$-dimensional velocity field. For simplicity, we drop the explicit noise argument and write $\boldsymbol{x}^\epsilon_s:=\mathcal{E}_{s,t}(\boldsymbol{x}_s)\equiv\mathcal{E}_{s,t}(\boldsymbol{x}_s, \boldsymbol{\epsilon})$.

\begin{restatable}[Expand Operator is Pushforward of Noise Law]{proposition}{pushforward}\label{prop:pushforward}
    Fix an expand operator $\mathcal{E}_{s,t}:\mathbb{R}^{d(s)}\times\mathbb{R}^{d(t)-d(s)}\to \mathbb{R}^{d(t)}$ that places source and noise coordinates according to one of the instantiations. Then, the resulting conditional law is:
    \begin{align}
        p_{t|s}(\cdot\mid\bx_s)=(\mathcal{E}_{s,t}(\bx_s, \cdot))_{\#}p_{\epsilon|s,t}(\cdot\mid\bx_s)\label{eq:expanding-cond-law}
    \end{align}
\end{restatable}

The proof is given in App \ref{app-prop:pushforward}. This proposition shows that the expand operator alone fixes the conditional law of the augmented state, but it does not yet describe how that state is transported to the target time. To complete the construction, we next specify how each newly inserted coordinate is denoised along the interpolant, which requires tracking the time elapsed since its insertion.

\paragraph{Local Time Coordinate} 
Since newly inserted coordinates are instantiated at different points along the interpolant, we assign each inserted component $i$ a \textit{local time coordinate} $t_i$, with the set of local times denoted $\boldsymbol{t}_{\text{local}}$. At the insertion time $t^{\text{ins}}_i$ of component $i$, the local time is initialized at $t_i=0$. At global time $t> t^{\text{ins}}_i$, it is defined as:
\begin{align}
    t_i=\frac{t-t^{\text{ins}}_i}{1-t^{\text{ins}}_i}\in [0,1]
\end{align}
This is a bijection that projects the global time onto a $[0,1]$ local time interval for each dimension $i$, so that every inserted dimension is denoised on a common local clock.

\paragraph{Transport Along the Expanding Interpolant}
After applying the expand operator, we integrate the standard velocity field of the continuous-time interpolant over $\tau\in [s,t]$ from the augmented state $\boldsymbol{x}_s^\epsilon$ to the target state $\boldsymbol{x}_t$:
\begin{align}
    \boldsymbol{x}_t=\boldsymbol{x}^\epsilon_s+\int_s^tb_\tau(\boldsymbol{x}_\tau^\epsilon)d\tau, \quad \boldsymbol{x}_t[i]=\boldsymbol{x}^\epsilon_s[i]+\int_{s_i}^{t_i}b_{\tau_i}(\boldsymbol{x}_{\tau_i}^\epsilon[i])d\tau_i, \quad b_t(\boldsymbol{x}_t^\epsilon)=\begin{cases}
        b_{t_i}(\boldsymbol{x}_{t_i}^\epsilon [i])& t_i^{\text{ins}}\leq t\\
        \boldsymbol{0}& t_i^{\text{ins}}> t
    \end{cases}\label{eq:flow-velocity-field}
\end{align}
where each coordinate $i$ is integrated along its own local time $\tau_i\in [s_i, t_i]$. Since $b_t: \mathbb{R}^{d(t)}\to \mathbb{R}^{d(t)}$ is a standard flow velocity field on the active coordinates ($t_i^{\text{ins}}\leq t$), we parameterize it as $\hat{b}_t$ by minimizing the standard CFM objective in (\ref{eq:cfm-loss}) and set it to zero otherwise, leaving the inactive coordinates unchanged. Given the expand operator and the transport ODE, we can define the \textit{Expanding Generative Flow} (EFlow).

\begin{definition}[Expanding Generative Flow]\label{def:expanding-flow}
    Let $0\leq s<t\leq 1$ with marginals $p_s\in \mathcal{P}(\mathbb{R}^{d(s)})$, $p_t\in \mathcal{P}(\mathbb{R}^{d(t)})$, expand operator $\mathcal{E}_{s,t}$ (\ref{eq:expand-operator}), and conditional noise $p_{\epsilon|s,t}(\cdot\mid\bx_s)\in \mathcal{P}(\mathbb{R}^{d(t)-d(s)})$. Write $\bx_s^\epsilon:=\mathcal{E}_{s,t}(\bx_s,\boldsymbol{\epsilon})$ and $p^{\epsilon}_{s,t}:=(\mathcal{E}_{s,t})_{\#}(p_s\otimes p_{\epsilon|s,t})$. An \textbf{expanding generative flow} from $p_s$ to $p_t$ is the composition
    \begin{align}
        \Phi_{s,t}:=X_{s,t}\circ\mathcal{E}_{s,t}:\mathbb{R}^{d(s)}\to\mathbb{R}^{d(t)},
    \end{align}
    where $X_{s,t}:\mathbb{R}^{d(t)}\to\mathbb{R}^{d(t)}$ is the time-$t$ flow map of the augmented interpolant ODE $\dot{\bx}^\epsilon_\tau=b_\tau(\bx^\epsilon_\tau)$ (\ref{eq:flow-velocity-field}) initialized at $\bx^\epsilon_s$ at time $s$. This map satisfies the marginal constraint $(X_{s,t})_{\#} p^\epsilon_{s,t}=p_t$, and being the deterministic flow from $\bx^\epsilon_s$, it fixes the full coupling between $\bx_s$ and $\bx_t$ rather than only the marginal $p_t$. The induced conditional law is
    \begin{align}
        p_{t|s}(\cdot\mid\bx_s)=\big(\Phi_{s,t}(\bx_s,\cdot)\big)_{\#} p_{\epsilon|s,t}(\cdot\mid\bx_s).\label{eq:cond-law}
    \end{align}
\end{definition}

In practice, we parameterize both the expand operator and the velocity field with a neural network. While $\Phi_{s,t}$ is a well-defined flow on the intervals between insertions, the jump discontinuity introduced at each insertion breaks the flow construction at those times. We therefore provide an alternative characterization of the expanding flow that is valid over the full interval $t\in [0,1]$.

\paragraph{Piecewise-Deterministic Markov Processes}
The expand operator and interpolant transport together specify a two-stage process, but not yet a single stochastic process with well-defined dynamics. We now show that these two stages can be unified as a \textit{piecewise-deterministic Markov process} (PDMP) \citep{davis1984piecewise} on the disjoint union of different-dimensional spaces $\mathcal{M}:=\bigsqcup_{t}\mathbb{R}^{d(t)}$. The PDMP is characterized by an extended generator, from which both the change-of-variables formula and the reduction to standard fixed-dimensional flows follow directly.

\begin{restatable}[Expanding Flows are Piecewise-Deterministic Markov Processes]{proposition}{pdmp}\label{prop:pdmp}
    Let $\mathcal{M}:=\bigsqcup_{t\in[0,1]}\mathbb{R}^{d(t)}$ be the state space and $(\boldsymbol{x}_{t})_{t\in [0,1]}$ the expanding generative flow of Def.~\ref{def:expanding-flow}, with velocity $b_t$ in the augmented space, expand operator $\mathcal{E}_t$, and noise law $p_{\epsilon \mid t}$. Writing $\mathcal{E}_{t}:=\lim_{s\to t}\mathcal{E}_{s,t}$ and $p_{\epsilon\mid t}:=\lim_{s\to t}p_{\epsilon\mid s,t}$ for the instantaneous insertion map and noise law at time $t$, the augmented process $\boldsymbol{S}_t:=(\boldsymbol{x}_t,\boldsymbol{t}_{\text{local}})$, carrying the local-time coordinates $\boldsymbol{t}_{\text{local}}$, is a piecewise-deterministic Markov process (PDMP) on $\mathcal{M}$ with extended generator acting on $\bx$-functions given by:
    \begin{align}
        \mathcal{L}_t f(\boldsymbol{x}) = \underbrace{b_t(\boldsymbol{x})\cdot\nabla f(\boldsymbol{x})}_{\text{transport}} + \underbrace{\lambda_t(\boldsymbol{x})\int\big(f(\mathcal{E}_t(\boldsymbol{x},\boldsymbol{\epsilon}))-f(\boldsymbol{x})\big)p_{\epsilon\mid t}(d\boldsymbol{\epsilon})}_{\text{jumps, kernel}=(\mathcal{E}_t(\boldsymbol{x},\cdot))_{\#}p_{\epsilon\mid t}}\label{eq:pdmp-generator},
    \end{align}
    where $\lambda_t$ is the insertion intensity and the jump kernel is the pushforward $(\mathcal{E}_t(\boldsymbol{x},\cdot))_{\#}p_{\epsilon\mid t}$, which equals the instantaneous conditional law $p_{t|s}(\cdot\mid\boldsymbol{x}_s)$ of (\ref{eq:expanding-cond-law}) as $s\uparrow t$.
\end{restatable}

The proof is provided in App \ref{app-prop:pdmp}. In the discrete setting, the insertion intensity takes an explicit form in terms of the per-gap insertion expectation. We defer the full characterization of the discrete case to Section \ref{sec:discrete-efm}.

\begin{restatable}[Change-of-Variables Formula]{corollary}{corollary:covformula}\label{corollary:cov}
    Given that $X_{s,t}$ is considered the time-$t$ map of an ODE in the augmented space $\frac{d}{d\tau}\boldsymbol{x}_\tau^\epsilon=b_\tau(\boldsymbol{x}_\tau^\epsilon)$, the marginal density can be defined with the instantaneous change-of-variables formula:
    \begin{align}
        \forall s\leq t,\quad \log p_t(X_{s,t}(\boldsymbol{x}_s^\epsilon))=\log p_s(\boldsymbol{x}^\epsilon_s)-\int_s^t\nabla\cdot b_\tau(X_{s,\tau}(\boldsymbol{x}^\epsilon_s))d\tau, \quad \nabla \cdot b_\tau =\sum_{i:t_i^{\text{ins}}\leq \tau}\partial_{\bx^i}b^i_\tau
    \end{align}
    where $b_\tau:\mathbb{R}^{d(t)}\to \mathbb{R}^{d(t)}$ is the instantaneous velocity field in the augmented space.
\end{restatable}

The divergence sums only over active coordinates ($t_i^{\text{ins}}\leq\tau$), so the density evolves as if it were already at its final dimension throughout, with non-inserted coordinates contributing nothing until their insertion time.

\section{Expanding Flow Maps}
\label{sec:expanding-flow-maps}
Given the definition of \textit{expanding generative flows}, we now define the \textbf{Expanding Flow Map} (EFM), which transports between two states along the expanding interpolant in a \textit{single} learned step, jointly increasing dimensionality and pushing the flow forward in time.

\paragraph{Flow Map in Augmented State Space}
We define the expanding flow map (EFM) $\Phi_{s,t}:\mathbb{R}^{d(s)}\to \mathbb{R}^{d(t)}$ as the composition of the expand operator $\mathcal{E}_{s,t}$ (\ref{eq:expand-operator}) which raises dimension from $d(s)$ to $d(t)$ and a transport map $X_{s,t}$ in $d(t)$-dimensional space that pushes the state forward along the \textit{local time coordinate} of each dimension. In this section, we define the transport map as the mean velocity $v_{s,t}$ for continuous state spaces, and in Section \ref{sec:discrete-efm}, we define it as the mean denoiser $\psi_{s,t}$ for discrete state spaces.

\begin{borderbox}
\begin{center}
\textbf{Expanding Flow Map}
\begin{align}
 \Phi_{s,t}=X_{s,t}\circ\mathcal{E}_{s,t}, \quad \Phi_{s,t}(\boldsymbol{x}_s):=X_{s,t}(\mathcal{E}_{s,t}(\boldsymbol{x}_s))=:\mathcal{E}_{s,t}(\boldsymbol{x}_s)+(t-s)v_{s,t}(\mathcal{E}_{s,t}(\boldsymbol{x}_s)), \quad \forall s,t\in [0,1]\label{eq:expanding-fm}
\end{align}
where $\mathcal{E}_{s,t}:\mathbb{R}^{d(s)}\to \mathbb{R}^{d(t)}$ is the expand operator and $v_{s,t}:\mathbb{R}^{d(t)}\to \mathbb{R}^{d(t)}$ is the mean velocity that parameterizes the transport map.
\end{center}
\end{borderbox}

Leveraging the \textit{expand operator} $\mathcal{E}_{s,t}$ from (\ref{eq:expand-operator}) which lifts the state $\boldsymbol{x}_s\in \mathbb{R}^{d(s)}$ to the dimensionality of the target state along the interpolant $\boldsymbol{x}_t\in \mathbb{R}^{d(t)}$, we can apply the standard residual form of the flow map via the average velocity $v_{s,t}$, which satisfies the \textit{tangent condition} \citep{kim2023consistency}:
\begin{align}
    \lim_{s\to t}\partial_t\Phi_{s,t}(\boldsymbol{x})=v_{t,t}(\boldsymbol{x})=b_t(\boldsymbol{x})
\end{align}
where $b_t$ is the standard flow velocity field defined in (\ref{eq:flow-velocity-field}). This condition can be enforced via the diagonal objective:
\begin{align}
    \mathcal{L}_{\text{diag}}(\hat{v}):=\int_0^1\mathbb{E}\big\|\hat{v}_{t,t}(\boldsymbol{x}_t)-\dot{\boldsymbol{x}}_t\big\|^2 dt,\quad \dot{\boldsymbol{x}}_t=
    \begin{cases}
        \boldsymbol{x}_1^{(t)}-\mathcal{E}_{0,t}(\boldsymbol{x}_0) & \text{(self-distillation from data)}\\[2pt]
        \hat{b}_t(\boldsymbol{x}_t) & \text{(distillation from a teacher flow)}
    \end{cases}
\end{align}
The intermediate state follows the expanding interpolant $\boldsymbol{x}_t=(1-t)\mathcal{E}_{0,t}(\boldsymbol{x}_0)+t\boldsymbol{x}_1^{(t)}$, where $\boldsymbol{x}_1^{(t)}$ is the clean sequence $\boldsymbol{x}_1$ truncated to the dimensions at time $t$ and the regression target is the time-derivative $\dot{\boldsymbol{x}}_t$.

\paragraph{Expanding Flow Map Identities}
In addition to the tangent condition, the expanding flow map $\Phi_{s,t}$ must also satisfy analogous \textit{consistency constraints}, which characterize a valid flow map as introduced in (\ref{eq:consistency-constraint}). Given our parameterization, we define these constraints with respect to the \textit{expand operator} $\hat{\mathcal{E}}_{s,t}$ and the average velocity $\hat{v}_{s,t}$ in the continuous-state setting.

\begin{restatable}[Expanding Flow Map Identities]{proposition}{expandingfm}\label{prop:expanding-fm-identities}
    Let $\boldsymbol{x}^\epsilon_s:= \mathcal{E}_{s,t}(\boldsymbol{x}_s)$. Then, the expanding flow map $\Phi_{s,t}(\boldsymbol{x}_s)$ satisfy the following consistency conditions:
    \begin{subequations}\label{eq:expanding-fm-identities}
    \begin{align}
        \text{Lagrangian:}\quad &\partial_t\Phi_{s,t}(\boldsymbol{x})=v_{t,t}(\Phi_{s,t}(\boldsymbol{x}))\\
        \text{Eulerian:}\quad &\partial_s\Phi_{s,t}(\boldsymbol{x})+ \nabla_{\boldsymbol{x}}\Phi_{s,t}(\boldsymbol{x})\tilde v_{s,s}(\boldsymbol{x})=0\\
        \text{Semigroup:}\quad &\Phi_{u,t}(\Phi_{s,u}(\boldsymbol{x}))=\Phi_{s,t}(\boldsymbol{x})
    \end{align}
    \end{subequations}
    where $\tilde{v}_{s,s}:\Real^{d(s)}\to \Real^{d(s)}$ is the unlifted diagonal velocity of \ref{ass:A4}, satisfying $v_{s,s}(\mathcal{E}_{s,t}(\bx))=E\tilde{v}_{s,s}(\bx)$. $\Phi_{s,t}(\boldsymbol{x}):=X_{s,t}(\mathcal{E}_{s,t}(\boldsymbol{x}))=\mathcal{E}_{s,t}(\boldsymbol{x})+(t-s)v_{s,t}(\mathcal{E}_{s,t}(\boldsymbol{x}))$ is the residual form of the expanding flow map in the continuous state space. 
\end{restatable}

These identities coincide with the fixed-dimensional flow map identities (\ref{eq:consistency-constraint}) with proof in \citet{boffi2025build}. The only difference is in the Eulerian condition, where the Jacobian $\nabla_{\bx}\Phi_{s,t}\in \Real^{d(t)\times d(s)}$ is rectangular, so multiplying with the unlifted velocity $\tilde{v}_{s,s}(\bx)\in \Real^{d(s)}$ of Assumption \ref{ass:A4} gives the diagonal velocity $v_{s,s}$ on the expanded state $\mathcal{E}_{s,t}(\boldsymbol{x})\in \mathbb{R}^{d(t)}$.  We provide proof in App \ref{app-prop:expanding-fm-identities}. To enforce these constraints, we define the analogous consistency objectives:
\begin{subequations}
\begin{small}
\begin{align}
    \mathcal{L}_{\text{LSD}}(\hat{\mathcal{E}}, \hat{v})&:=\int_0^1\int_0^t\mathbb{E}\left\|\partial_t\hat{\Phi}_{s,t}(\boldsymbol{x}_s)-\text{sg}\left(\hat{b}_t(\hat{\Phi}_{s,t}(\boldsymbol{x}_s))\right)\right\|^2dsdt+\mathcal{L}_{\text{diag}}(\hat{v})\\
    \mathcal{L}_{\text{ESD}}(\hat{\mathcal{E}}, \hat{v})&:=\int_0^1\int_0^t\mathbb{E}\left\|\partial_s\hat{\Phi}_{s,t}(\boldsymbol{x}_s)+\text{sg}\left(\nabla_{\boldsymbol{x}}\hat{\Phi}_{s,t}(\boldsymbol{x}_s) \hat{b}_s(\boldsymbol{x}_s)\right)\right\|^2dsdt+\mathcal{L}_{\text{diag}}(\hat{v})\\
    \mathcal{L}_{\text{PSD}}(\hat{\mathcal{E}}, \hat{v})&:=\int_0^1\int_0^t\int_s^t\mathbb{E}\big\|\hat{\Phi}_{s,t}(\boldsymbol{x}_s)-\text{sg}(\hat{\Phi}_{u,t}(\hat{\Phi}_{s,u}(\boldsymbol{x}_s)))\big\|^2dudsdt+\mathcal{L}_{\text{diag}}(\hat{v})
\end{align}
\end{small}
\end{subequations}
where unlike fixed-dimensional flow maps, $\hat{\Phi}_{s,t}$ is parameterized via both $\hat{\mathcal{E}}_{s,t}$ and $\hat{v}_{s,t}$. 

\paragraph{Stochastic Flow Map Interpretation}
A notable property of the expanding flow map is that it defines a \textit{stochastic flow} whose randomness arises from an explicit conditional noise injection. Unlike deterministic flow maps, this construction induces a \textit{distribution} over trajectories by augmenting the state with latent noise and transporting it through the flow. We formalize this idea in the following proposition.

\begin{restatable}[Expanding Flow Maps are Stochastic Flow Maps]{proposition}{stochasticfm}\label{prop:stochasticfm}
    Let $\Phi_{s,t}(\bx_s,\boldsymbol{\epsilon})$ be an expanding flow map (\ref{eq:expanding-fm}) with conditional noise $\boldsymbol{\epsilon }\sim p_{\epsilon|s,t}$ satisfying the consistency identities (\ref{eq:expanding-fm-identities}). Then for all $t\in [s,1]$, the pushforward satisfies $p_{t|s}(\cdot |\boldsymbol{x}_s)=(\Phi_{s,t}(\boldsymbol{x}_s, \cdot))_{\#}p_{\epsilon|s,t}$, and recovers the clean posterior $p_{1|s}(\cdot|\boldsymbol{x}_s)$ at $t=1$.
\end{restatable}

The proof is provided in App \ref{app-prop:stochasticfm}. This property follows the construction of the expanding flow map as the composition of learned expanding and denoising operators that reconstruct the interpolant at a future time step $t> s$, capturing the full posterior over future states rather than committing to a single deterministic trajectory. 

\section{Discrete Expanding Flow Maps}
\label{sec:discrete-efm}
The expanding generative flows and expanding flow maps framework is naturally positioned for \textit{variable-length discrete generation}, where a sequence or graph dynamically expands from a single token or node into a full sequence or graph. We instantiate the objects of Sections~\ref{sec:expanding-flows} and \ref{sec:expanding-flow-maps} on the discrete simplex: the expand operator $\mathcal{E}_{s,t}$ becomes \textit{token insertions} and the transport $X_{s,t}$ becomes \textit{token denoising}.

\paragraph{Discrete Expanding Interpolant for Growing Sequence Lengths}
We first define the \textbf{discrete expanding interpolant}, in which the number of tokens $d(t):[0,1]\to\mathbb{N}$ grows over intermediate times. At training time, we fix an \textbf{insertion schedule} $\alpha_t$ interpolating between $\alpha_0=0$ and $\alpha_1=1$. An intermediate sequence $\boldsymbol{x}_t$ is constructed by first sampling a clean sequence from the data distribution, $\boldsymbol{x}_1\sim p_1$, and then independently sampling an insertion time $t^{\text{ins}}_i$ for each position $i$ according to the schedule. Each position is either not yet inserted, denoted (empty), or in an intermediate latent state $\boldsymbol{x}_{t_i}^i$ along the continuous-time flow at its rescaled local time $t_i$:
\begin{align}
    \Pr[t^{\text{ins}}_i\leq t]=\alpha_t, \quad
    \boldsymbol{x}_t^i=\begin{cases}
        \text{(empty)}, & t< t^{\text{ins}}_i\\
        (1-t_i)\boldsymbol{x}_0^i+t_i\boldsymbol{x}_1^i, & t\geq t^{\text{ins}}_i
    \end{cases},
    \quad t_i=\frac{t-t^{\text{ins}}_i}{1-t^{\text{ins}}_i}
\end{align}
A sequence along the interpolant at global time $t$ is therefore the subsequence of currently inserted tokens:
\begin{align}
    \boldsymbol{x}_t=(\boldsymbol{x}^1_{t_1},\dots,\boldsymbol{x}^{d(t)}_{t_{d(t)}})\in\mathbb{R}^{d(t)\times V},
\end{align}
where each token $\boldsymbol{x}^i_{t_i}\in\mathbb{R}^V$ is a latent vector denoised to its local time $t_i$.

\paragraph{Parameterizing Token Insertions}
Given an intermediate subsequence $\bx_s$ and a target time $t>s$, we decide where and how many $V$-dimensional latent tokens to insert on the jump via the (two-time) \textbf{insertion expectation} $\mathcal{I}_{s,t}(\bx_s)[i]$, the expected number of tokens inserted into gap $i$ (the slot between active tokens $\bx_s^{i-1}$ and $\bx_s^i$) over $(s,t]$:
\begin{align}
    \mathcal{I}_{s,t}(\boldsymbol{x}_s)[i]=\tfrac{\alpha_t-\alpha_s}{1-\alpha_s}\mathcal{I}_s(\boldsymbol{x}_s)[i],\qquad \mathcal{I}_t(\boldsymbol{x}_t)[i]:=\mathbb{E}\big[g_i(t)\mid\boldsymbol{x}_t\big],\quad g_i(t):=\text{ind}_i (t) -\text{ind}_{i- 1}(t)-1\label{eq:insert-expectation}
\end{align}
where $g_i(t)$ is the count still missing from gap $i$ at $t$, $\text{ind}_i(t)$ the index of active position $i$ in $\bx_1$, and $\tfrac{\alpha_t-\alpha_s}{1-\alpha_s}$ is the fraction of the remaining tokens inserted over $(s,t]$ with the insertion hazard $\lambda_t:=\tfrac{\dot{\alpha_t}}{1-\alpha_t}$ defining the instantaneous insertion rate of an uninserted token at time $t$. The insertion head $\hat{\mathcal{I}}_{s,t}(\bx_s)$ is conditioned on both endpoints and trained on diagonal and off-diagonal time pairs with the insertion loss:
\begin{align}
    \mathcal{L}_{\text{insert}}(\hat{\mathcal{I}})=\mathbb{E}_{t}\mathbb{E}_{\bx_0,\bx_1}\Big[\tfrac{\dot{\alpha}_t}{1-\alpha_t}\textstyle\sum_{i}\phi\big(g_i(t), \hat{\mathcal{I}}_t(\bx_t)[i]\big)\Big]+\mathbb{E}_{s< t}\mathbb{E}_{\bx_0,\bx_1}\Big[\textstyle\sum_{i}\phi\big(\mu_i(s,t), \hat{\mathcal{I}}_{s,t}(\bx_s)[i]\big)\Big]\label{eq:insert-loss}
\end{align}
where $\phi(a,b):=b-a+a\log\tfrac{a}{b}$ is the Poisson NLL of count $a$ under mean $b$ up to a $b$-independent constant, so $\arg\min_b\mathbb{E}[\phi(A,b)]=\mathbb{E}[A]$. Both terms supervise against a realized count and recover its conditional mean: the diagonal term supervises against the gap count $g_i(t)$, giving $\hat{\mathcal{I}}_t=\mathcal{I}_t$ and the off-diagonal term supervises against the interval count $\mu_i(s,t):=\#\{j:t^{\text{ins}}_j\in(s,t],j\in\text{gap } i\}$, whose mean is $\mathcal{I}_{s,t}(\bx_s)[i]$, so the insertion head learns the two-time count directly.

\paragraph{Discrete Expand Operator}
Given the learned insertion expectation, the \textit{discrete expand operator} $\mathcal{E}_{s,t}$ inserts latent tokens $\boldsymbol{x}_0\sim \mathcal{N}(\boldsymbol{0}, \sigma^2\boldsymbol{I}_V)$ with prior scale $\sigma$ and local time $t_i=0$ per gap according to the gap-lengths $\ell_i$:
\begin{align}
    \mathcal{E}_{s,t}(\bx_s):=\bx^\epsilon_s=\left[\bigoplus_{i=1}^{d(s)}\big(\bx_0^{\ell_i}\oplus\bx_{s_i}^i\big)\right]\oplus\bx_0^{\ell_{d(s)+1}}
\end{align}
Since the head predicts the interval count $\hat{\mathcal{I}}_{s,t}(\bx_s)[i]$ directly, we draw $\ell_i$ from the exact conditional law of the interpolant. Since each of the $L-n_s$ uninserted tokens lands in gap $i$ over $(s,t]$ independently, we have:
\begin{align}
    \ell_i \sim \text{Binomial}\left(L - n_s,\ \tfrac{\hat{\mathcal I}_{s,t}(\bx_s)[i]}{L - n_s}\right)\label{eq:insertion-dist}
\end{align}
which matches the predicted mean, $\mathbb{E}[\ell_i]=\hat{\mathcal{I}}_{s,t}(\bx_s)[i]$, while capping each gap at the remaining budget; we truncate the joint draw left-to-right so that $\sum_i\ell_i\leq L-n_s$. Writing $\rho_{s,t}:=\tfrac{\alpha_t-\alpha_s}{1-\alpha_s}$ for the fraction of tokens inserted over $s\to t$, the size of the jump determines how much this cap matters. 

\begin{wrapfigure}{r}{0.4\textwidth}
    \includegraphics[width=\linewidth]{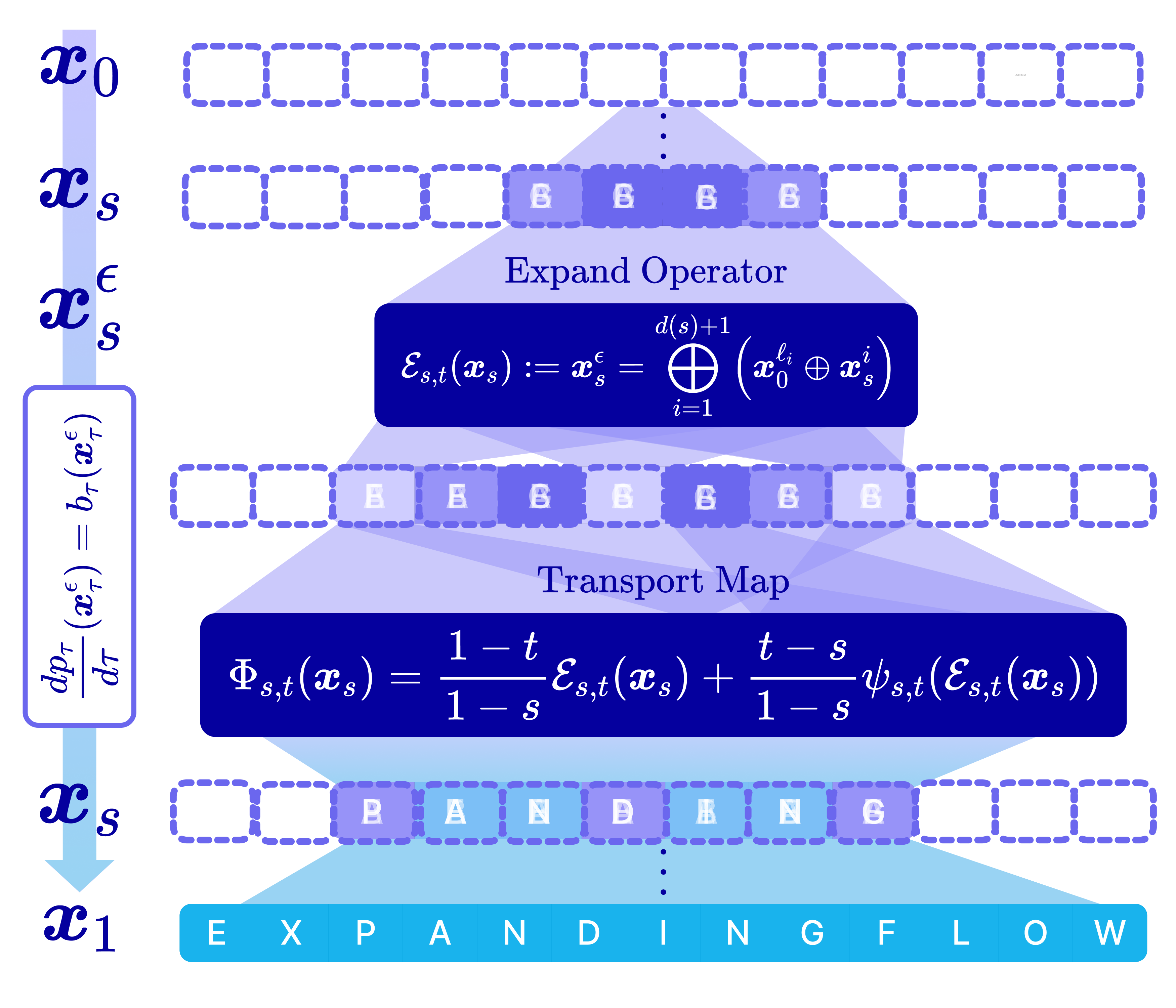}
    \caption{\textbf{Discrete Expanding Flow Maps.} Latent tokens are inserted via the expand operator $\mathcal{E}_{s,t}$ and denoised via the transport map defined with the mean denoiser $\psi_{s,t}$.}
    \label{fig:discrete-efm}
\end{wrapfigure}

Many-step sampling takes $\rho_{s,t}\ll 1$, inserting a few tokens per step against a budget of hundreds, so $L-n_s\gg\hat{\mathcal{I}}_{s,t}$ and (\ref{eq:insertion-dist}) converges to the Poisson operator $\ell_i\sim\text{Poisson}(\hat{\mathcal{I}}_{s,t}(\bx_s)[i])$, the infinitesimal jump law of the PDMP of Section~\ref{sec:expanding-flows} with insertion intensity $\lambda_t:=\tfrac{\dot\alpha_t}{1-\alpha_t}\sum_i\mathcal{I}_t(\bx_t)[i]$. 

For few-step sampling where the jump interval approaches the boundary pair $(s,t)=(0,1)$, the operator must produce the whole sequence in one jump, with $\rho_{s,t}\to 1$ and $\sum_i\hat{\mathcal{I}}_{s,t}[i]\approx L$. In this setting, the Poisson limit fails with unbounded cap and high variance equal to its mean, while the binomial concentrates to the exact count with variance $np(1-p)$ for $\text{Binomial}(n,p)$ vanishing as $p\to1$.

\paragraph{Discrete Expanding Flow Map}
Following the \textit{mean denoiser} (\ref{eq:mean-denoiser}) of fixed-length discrete flow maps, we define the \textit{expanding mean denoiser} as:
\begin{align}
    \psi_{s,t}(\mathcal{E}_{s,t}(\boldsymbol{x}_s)):=\mathcal{E}_{s,t}(\boldsymbol{x}_s)+(1-s)v_{s,t}(\mathcal{E}_{s,t}(\boldsymbol{x}_s))
\end{align}
Substituting this into the definition of the expanding flow map in (\ref{eq:expanding-fm}) yields the \textit{discrete expanding flow map} (EFM), which jointly inserts and denoises tokens in a discrete sequence toward the target one-hot distribution.

\begin{borderbox}
\begin{center}
\textbf{Discrete Expanding Flow Map}
\begin{align}
    \Phi_{s,t}(\boldsymbol{x}_s):= X_{s,t}(\mathcal{E}_{s,t}(\boldsymbol{x}_s))=\frac{1-t}{1-s}\mathcal{E}_{s,t}(\boldsymbol{x}_s)+\frac{t-s}{1-s}\psi_{s,t}(\mathcal{E}_{s,t}(\boldsymbol{x}_s))
\end{align}
where $\mathcal{E}_{s,t}(\boldsymbol{x}_s):\mathbb{R}^{d(s)\times V}\to  \mathbb{R}^{d(t)\times V}$ is the expand operator from sequence length $d(s)$ to $d(t)$ and vocabulary size $V$ and $\psi_{s,t}:\mathbb{R}^{d(t)\times V}\to (\Delta^{V-1})^{d(t)}$ is the mean denoiser that parameterizes the transport map.
\end{center}
\end{borderbox} 

Unlike the continuous-space definition of EFM, the mean denoiser that parameterizes discrete EFM is defined on the simplex at each position $\psi_{s,t}(\bx_s)[i]\in \Delta^{V-1}$ and can be optimized with cross-entropy objectives.

\paragraph{Discrete Expanding Flow Map Identities}
To parameterize discrete EFM, we extend the continuous-time consistency identities (\ref{eq:expanding-fm-identities}) to the discrete setting with respect to the mean denoiser on the expanded state $\psi_{s,t}(\mathcal{E}_{s,t}(\boldsymbol{x}))$, extending prior work on fixed-length discrete flow maps \citep{lee2026flow, potaptchik2026discrete}. 

\begin{restatable}[Discrete Expanding Flow Map Identities]{proposition}{discreteefm}\label{prop:discrete-efm}
    Let $\boldsymbol{x}^\epsilon_s:= \mathcal{E}_{s,t}(\boldsymbol{x}_s)$. Then, for $s<t$ (and $s<u<t$ in the semigroup case), the expanding flow map $\Phi_{s,t}(\boldsymbol{x}_s)$ and mean denoiser $\psi_{s,t}(\boldsymbol{x}^\epsilon_s)$ satisfy the following consistency conditions:
    \begin{subequations}\label{eq:discrete-consistency}
    \begin{align}
        \text{Lagrangian:}\quad &\psi_{s,t}(\boldsymbol{x}_s^\epsilon)=\psi_{t,t}(\Phi_{s,t}(\boldsymbol{x}_s))-\gamma_{s,t}\partial_t\psi_{s,t}(\boldsymbol{x}^\epsilon_s)\\
        \text{Eulerian:}\quad &\partial_s\psi_{s,t}(\boldsymbol{x}^\epsilon_s)+J_x\psi_{s,t}(\boldsymbol{x}^\epsilon_s)b_s(\boldsymbol{x}^\epsilon_s)=\kappa_{s,t}(\psi_{s,t}(\boldsymbol{x}^\epsilon_s)-\mathcal{E}_{s,t}(\psi_{s,s}(\boldsymbol{x}_s)))\\
        \text{Semigroup:}\quad &\psist(\xseps)=\omega_{s,u,t} \psi^{(t)}_{s,u}(\xseps)+(1-\omega_{s,u,t}) \psi_{u,t}\big(\E_{u,t}(\Phi_{s,u}(\bx_{s}))\big)
    \end{align}
    \end{subequations}
    where $J_x\psi_{s,t}(\bx^\epsilon_s)$ denotes the Jacobian of the map $\psi_{s,t}$ with respect to its $d(t)$-dimensional input, and the coefficients are defined as $\omega_{s,u,t}:=\frac{(u-s)(1-t)}{(t-s)(1-u)}$, $1-\omega_{s,u,t}=\frac{(t-u)(1-s)}{(t-s)(1-u)}$, $\gamma_{s,t}:=\frac{(1-t)(t-s)}{1-s}$, and $\kappa_{s,t}:=\frac{1-t}{(1-s)(t-s)}$. We define $\psi^{(t)}_{s,u}:\Real^{d(t)\times V}\to\Real^{d(t)\times V}$ as the mean denoiser over $[s,u]$ evaluated at the augmented dimension $d(t)$, the same parameterized map as $\psi_{s,u}$ applied to a longer input. Because the velocity mask of (\ref{eq:flow-velocity-field}) enforces \ref{ass:A3}, $\psi^{(t)}_{s,u}$ leaves the coordinates inserted over $(u,t]$ unchanged, so $\psi^{(t)}_{s,u}(\xseps)=\E_{u,t}(\psi_{s,u}(\E_{s,u}(\bx_s)))$.
\end{restatable}

The proof is provided in App \ref{app-prop:discrete-efm}. These properties define consistency objectives for training $\hat{\mathcal{E}}$ and $\hat{\psi}$. Since $\hat{\psi}:\mathbb{R}^{d(t)\times V}\to (\Delta^{V-1})^{d(t)}$, we follow \citep{lee2026flow} and define the discrete EFM objective in the cross-entropy form:
\begin{small}
\begin{align}
    \mathcal{L}_{\text{DEFM}}( \hat{\psi}):=&\mathbb{E}_{s,u,t}\mathbb{E}_{\boldsymbol{x}_0, \boldsymbol{x}_1}\bigg[-\sum_{i=1}^{d(t)}(\bar{\psi}_{s,t}\cdot \log \hat{\psi}_{s,t})(\boldsymbol{x}_s^\epsilon)_i\bigg]+\mathbb{E}_t\mathbb{E}_{\boldsymbol{x}_0, \boldsymbol{x}_1}\bigg[-\sum_{i=1}^{d(t)}(D_t\cdot \log \hat{\psi}_{t,t})(\boldsymbol{x}_t^i)\bigg]\label{eq:cross-entropy-loss}
\end{align}
\end{small}
where $D_t:\mathbb{R}^{d(t)\times V}\to  (\Delta^{V-1})^{d(t)}$ is either a teacher instantaneous denoiser or the diagonal values of the student mean denoiser. Both terms in the target act on the state $\boldsymbol{x}^\epsilon_s=\mathcal{E}_{s,t}(\boldsymbol{x}_s)$ expanded once to the terminal length $d(t)$, with positions inserted after $u$ left empty:
\begin{align}
    \bar{\psi}_{s,t}(\boldsymbol{x}^\epsilon_s):=\text{sg}\left[\omega_{s,u,t} \hat{\psi}^{(t)}_{s,u}(\boldsymbol{x}^\epsilon_s)+(1-\omega_{s,u,t}) \hat{\psi}_{u,t}\big(\hat{\Phi}^{(t)}_{s,u}(\boldsymbol{x}^\epsilon_s)\big)\right]\label{eq:semigroup-target}
\end{align}
which enforces the discrete semigroup condition in Prop \ref{prop:discrete-efm} as shown in App \ref{app-prop:discrete-efm} and, together with $\mathcal{L}_{\text{insert}}$, trains the mean denoiser and expand operator for few-step generation in the discrete setting. The training procedures for discrete EFlow and EFM are provided in Algorithms \ref{alg:training-eflow} and \ref{alg:training-efm}, respectively.

\section{Experiments}
\label{sec:experiments}
To establish EFlow and EFM as a generalizable framework for generative modeling with variable canvas sizes across continuous and discrete state spaces, we evaluate our approach on three diverse tasks: coarse-to-fine molecular conformer generation on continuous 3D coordinate space (Section \ref{exp:conformer}), discrete molecular graph generation (Section \ref{exp:graph}), and language modeling (Section \ref{exp:language}).

\begin{table}[t]
\centering
\caption{\textbf{Molecular conformer generation results on GEOM-Drugs and GEOM-QM9}. Comparison against state-of-the-art multi-step diffusion and deep generative models. +R indicates sampling with 10 additional steps, using a learned coordinate-refinement network. Reported baseline values are taken from \citep{huhierarchical}. Best values are \textbf{bolded} and second-best are \underline{underlined}.}
\label{table:conformer-results}
\resizebox{\textwidth}{!}{%
\begin{tabular}{l c cccc cccc}
\toprule
& & \multicolumn{4}{c}{\textbf{GEOM-QM9} ($\delta=0.5$ \AA)} & \multicolumn{4}{c}{\textbf{GEOM-Drugs} ($\delta=1.25$ \AA)} \\
\cmidrule(lr){3-6} \cmidrule(lr){7-10}
& & COV-R(\%) $\uparrow$ & AMR-R(\AA) $\downarrow$ & COV-P(\%) $\uparrow$ & AMR-P(\AA) $\downarrow$ & COV-R(\%) $\uparrow$ & AMR-R(\AA) $\downarrow$ & COV-P(\%) $\uparrow$ & AMR-P(\AA) $\downarrow$ \\
Models & Steps $\downarrow$ & Mean/Median & Mean/Median & Mean/Median & Mean/Median & Mean/Median & Mean/Median & Mean/Median & Mean/Median \\
\midrule
GraphDG        & 500  & 73.33/84.21 & 0.4245/0.3973 & 43.90/35.33 & 0.5809/0.5823 & 8.27/0.00 & 1.9722/1.9845 & 2.08/0.00 & 2.4340/2.4100 \\
CGCF           & 500  & 78.05/82.48 & 0.4219/0.3900 & 36.49/33.57 & 0.6615/0.6427 & 53.96/57.06 & 1.2487/1.2247 & 21.68/13.72 & 1.8671/1.8066 \\
ConfVAE        & 500  & 55.84/88.20 & 0.4154/0.3739 & 38.02/34.67 & 0.6215/0.6091 & 55.20/59.43 & 1.2380/1.1417 & 22.96/14.05 & 1.8287/1.8159 \\
GeoMol         & 500  & 71.26/72.00 & 0.3731/0.3731 & -/- & -/- & 67.16/71.71 & 1.0875/1.0586 & -/- & -/- \\
ConfGF         & 500  & 88.49/94.31 & 0.2673/0.2685 & 46.43/43.41 & 0.5224/0.5124 & 62.15/70.93 & 1.1629/1.1596 & 23.42/15.52 & 1.7219/1.6863 \\
\midrule
GeoDiff        & 500  & 87.80/93.66 & 0.3179/0.3216 & 46.25/45.02 & 0.6173/0.5112 & 64.12/75.56 & 1.1444/1.1296 & 43.16/42.02 & 1.3806/1.3314 \\
SubgDiff       & 500  & \textbf{89.40}/\underline{94.39} & \textbf{0.2543}/\underline{0.2601} & 49.21/47.45 & 0.5030/0.4724 & 74.30/77.87 & 1.0003/0.9905 & -/- & -/- \\
MSGEN  & 1000 & 88.36/\textbf{95.10} & \underline{0.2606}/\textbf{0.2332} & 49.54/48.39 & 0.4934/\textbf{0.4624} & 85.06/\textbf{97.16} & 0.9538/0.9478 & 49.28/48.20 & 1.3140/1.2719 \\
\midrule
\rowcolor{mybg} EFlow & 20 & 88.51/93.01 & 0.3849/0.3585 & 53.07/52.27 & 0.5185/0.4891 & 81.78/90.91 & 0.9971/0.9826 & 55.88/61.58 & 1.2672/1.2028 \\
\rowcolor{mybg} EFlow+R  & 30 & \underline{88.87}/94.05 & 0.3212/0.2924 & \underline{54.86}/\underline{52.38} & \textbf{0.4876}/\textbf{0.4636} & \underline{85.80}/93.85 & \underline{0.9069}/\underline{0.8855} & \underline{64.05}/\underline{70.38} & \underline{1.1663}/\underline{1.0997} \\
\midrule
\rowcolor{mybg} EFM & 1 & 69.21/69.23 & 0.4884/0.4635 & 32.18/29.09 & 0.6093/0.5794 & 63.28/72.47 & 1.1921/1.1721 & 31.03/26.09 & 1.4966/1.4454 \\
\rowcolor{mybg} EFM & 2 & 72.83/74.25 & 0.4703/0.4436 & 34.09/31.33 & 0.6000/0.5654 & 77.21/88.67 & 1.0740/1.0499 & 47.14/49.24 & 1.3483/1.2966 \\
\rowcolor{mybg} EFM & 4 & 83.62/86.91 & 0.4165/0.3954 & 48.76/45.26 & 0.5412/0.5159 & 81.26/91.13 & 1.0201/1.0029 & 53.37/56.72 & 1.2883/1.2354 \\
\rowcolor{mybg} EFM+R & 14 & 87.79/92.01 & 0.3271/0.3049 & \textbf{55.56}/\textbf{53.09} & \underline{0.4889}/\underline{0.4691} & \textbf{87.97}/\underline{95.45} & \textbf{0.9029}/\textbf{0.8836} & \textbf{64.67}/\textbf{73.25} & \textbf{1.1575}/\textbf{1.0875} \\

\bottomrule
\end{tabular}%
}
\end{table}

\subsection{Coarse-to-Fine Molecular Conformer Generation}
\label{exp:conformer}
\paragraph{Setup and Training Framework}
We evaluate our framework on molecular conformer generation over the GEOM-QM9 and GEOM-Drugs datasets \citep{axelrod2022geom}, following the GeoDiff split \citep{xu2022geodiff}, and report Coverage (COV) and Average Minimum RMSD (AMR) in both recall (R) and precision (P) forms at thresholds $\delta=0.5$ \AA{} for QM9 and $\delta=1.25$ \AA{} for Drugs (Table \ref{table:conformer-results}; see App.~\ref{app:conformer} for experimental details). EFlow treats a conformer as a coarse-to-fine continuous-space expanding interpolant: the heavy-atom backbone is denoised from Gaussian noise and frozen at an insertion-end time, after which the hydrogens are inserted and resolved around the fixed backbone. Both stages live in a \emph{single} insertion flow with per-atom local times, so one network generates the all-atom conformer, and an optional low-noise refiner (EFlow+R, EFM+R) refines the result. For fair comparison, the velocity field is the GeoDiff dual-encoder \citep{xu2022geodiff} used by MSGEN \citep{huhierarchical}. Since the molecular graph fixes the atom count, the expand operator here is the deterministic case of child expansion, where heavy atoms are held fixed, and hydrogen atoms are sampled from a deterministic insertion schedule but conditioned on the positions of their heavy-atom parent.  

\paragraph{Results}
We compare against strong multi-step diffusion baselines: GeoDiff \citep{xu2022geodiff}, SubgDiff \citep{zhang2024subgdiff}, and the two-stage MSGEN \citep{huhierarchical}, which run $500$-$1000$ denoising steps, alongside the classical and generative baselines (App \ref{app:conformer}). Despite using $16$-$50\times$ fewer steps than the diffusion baselines, EFlow is competitive with or surpasses them across both datasets. On GEOM-QM9, EFlow attains COV-R $88.51\%$ and COV-P $53.07\%$ at $20$ steps, and the refiner sharpens precision to the best AMR-P of any method ($0.4876$ \AA{}) at $30$ steps. The gains are most significant on the larger GEOM-Drugs molecules (up to 181 atoms), where our methods take the best \textit{and} second-best value on all four metrics. The refiner consistently sharpens the precision metrics with only $10$ additional evaluations by refining local geometry around the frozen heavy-atom backbone. Notably, EFlow without distillation already surpasses the performance of multi-step diffusion at a fraction of the budget. After distillation, EFM pushes generation into the few-step regime, producing conformers in as few as $4$, $2$, and even $1$ step, with EFM+R attaining the single best score on all four GEOM-Drugs metrics. To evaluate the chemical validity of generated samples, we evaluate the molecular ensemble properties \citep{axelrod2022geom} for 30 molecules in GEOM-QM9 on 50 generated conformers using Psi4 \citep{smith2020psi4} and report these results in Table \ref{table:conformer-property-median}. These results establish expanding flows as a single framework that scales from multi-step to few-step conformer generation.

\subsection{Discrete Molecular Graph Generation}
\label{exp:graph}
\paragraph{Setup and Training Framework}
We evaluate our framework on drug-like small molecules from the QM9 dataset \citep{wu2018moleculenet}, reporting validity, uniqueness, and Fréchet ChemNet Distance (FCD) \citep{preuer2018frechet} on $10{,}000$ generated molecules (Table \ref{table:molecule-results}; see App \ref{app:molecule} for experimental details). EFlow treats a molecule as a discrete expanding interpolant over a variable number of nodes. The state is a pair $(\boldsymbol{x}, \boldsymbol{E})$ of node categories and a symmetric edge-category matrix, where each node has an independent insertion time, and a learned insertion head predicts the per-gap node count, enabling the graph to grow from empty to its full size. In contrast, node and edge categories are denoised jointly by a single network. For fair comparison, the denoiser is the DeFoG graph transformer \citep{qin2024defog}. In contrast to conformer generation, the graph size is unknown, and the expand operator is a learned head predicting how many nodes to insert at sampling time (see App \ref{app:graph-algo} for the discrete graph implementation details).

\begin{wraptable}{r}{0.55\textwidth}
\vspace{-\intextsep}
\caption{\textbf{Discrete molecular graph generation results.} Comparison against state-of-the-art multi-step baseline DeFoG \citep{qin2024defog} (top) and flow map baseline CFM \citep{roos2026categorical} (bottom) on 10,000 molecules with FCD computed on the test split. Best FCD values are \textbf{bolded}.}
\label{table:molecule-results}
\begin{center}
\begin{small}
\resizebox{\linewidth}{!}{
\begin{tabular}{lccc|>{\columncolor{mybg}}c>{\columncolor{mybg}}c>{\columncolor{mybg}}c}
\toprule
\textit{Multi-Step} & \multicolumn{3}{c|}{\textbf{DeFoG}} & \multicolumn{3}{c}{\textbf{EFlow} (ours)} \\
\cmidrule(lr){2-4} \cmidrule(lr){5-7}
\textbf{Steps} & \textbf{Valid} $\uparrow$ & \textbf{Unique} $\uparrow$ & \textbf{FCD} $\downarrow$
 & \textbf{Valid} $\uparrow$ & \textbf{Unique} $\uparrow$ & \textbf{FCD} $\downarrow$ \\
\midrule
100 & 99.2 & 96.2 & 0.134 & 98.9 & 95.9 & \textbf{0.116} \\
50  & 98.8 & 96.2 & 0.260 & 98.8 & 96.1 & \textbf{0.131} \\
10 & 88.6 & 92.8 & 1.630 & 97.9 & 96.1 & \textbf{0.390} \\
4  & 53.6 & 63.1 & 3.008 & 91.7 & 92.8 & \textbf{1.780} \\
\midrule
\textit{Few-Step} & \multicolumn{3}{c|}{\textbf{CFM}} & \multicolumn{3}{c}{\textbf{EFM} (ours)} \\
\cmidrule(lr){2-4} \cmidrule(lr){5-7}
\textbf{Steps} & \textbf{Valid} $\uparrow$ & \textbf{Unique} $\uparrow$ & \textbf{FCD} $\downarrow$
 & \textbf{Valid} $\uparrow$ & \textbf{Unique} $\uparrow$ & \textbf{FCD} $\downarrow$ \\
\midrule
2 & 91.1 & 97.6 & 0.49 & 90.0 & 97.9 & \textbf{0.40} \\
1 & 90.0 & 91.8 & 2.14 & 90.2 & 97.7 & \textbf{0.44} \\
\bottomrule
\end{tabular}
}
\end{small}
\end{center}
\vspace{-10pt}
\end{wraptable}
\paragraph{Results}
First, we evaluate EFlow without any distillation against the state-of-the-art discrete graph flow model DeFoG \citep{qin2024defog}. At a $100$-step budget, EFlow already surpasses DeFoG with lower FCD, with greater improvement as the step budget shrinks. From $100$ to $10$ steps, EFlow's FCD degrades by only slightly while validity and uniqueness remain above $95\%$. DeFoG's FCD and validity degrade significantly over the same range, falling to $53.6\%$ at $4$ steps, while EFlow still produces $91.7\%$ valid and $92.8\%$ unique molecules at FCD $1.780$. Notably, this is the non-distilled EFlow teacher, suggesting meaningful room for further gains in the few-step regime.

We then evaluate the distilled EFM student against categorical flow maps (CFM) \citep{roos2026categorical} in the one- to two-step regime (Table \ref{table:molecule-results}). In two steps, EFM is comparable to CFM on validity and uniqueness and achieves a lower FCD ($0.40$ vs.\ $0.49$). EFM's advantage widens at a single step, where CFM's FCD degrades to $2.14$, and uniqueness drops to $91.8$, while EFM's FCD remains strong at $0.44$ with $97.7\%$ uniqueness. At one step, EFM's FCD is nearly $5\times$ lower than CFM's ($0.44$ vs.\ $2.14$). Together, these results establish expanding flows as a single framework robust across multi-step and one-step generation of variable-sized graphs.

\subsection{Variable-Length Language Modeling}
\label{exp:language}

\paragraph{Setup and Training Framework}
Finally, we evaluate unconditional text generation on the One Billion Word benchmark (LM1B) \citep{chelba2013one}, reporting generative perplexity under GPT-2-Large \citep{radford2019language} and the entropy of the generated token distribution across both the multi-step and few-step regimes (Tables \ref{table:language-results} and \ref{table:efm-language-results}; see App \ref{app:language} for experimental details). We define a discrete expanding interpolant on sentences, where each token is assigned an independent insertion time, and a learned head predicts per-gap insertions, allowing the sequence to grow from empty to full length as a single network denoises token categories. Sampling either integrates the learned velocity via the instantaneous denoiser or mean denoiser and takes the argmax only at the final step. Due to the significantly higher vocabulary size of language data ($V=30,522$ for LM1B), we restrict the insertion times to the interval $[0,t_{\text{ins-end}}]$ to allow sufficient denoising time. The algorithmic choices specific to language are discussed in App \ref{app:discrete-efm-implementation} and experiment details are provided in App \ref{app:language}. 

\begin{table}[t]
\centering
\begin{minipage}[t]{0.42\textwidth}
\centering
\caption{\textbf{EFlow performance on LM1B.} Gen.\ PPL ($\downarrow$) and entropy across sampling budgets. FLM results are computed by running publicly released checkpoints from \citet{lee2026flow}. Best Gen. PPL values are \textbf{bolded}.}
\label{table:language-results}
\begin{small}
\resizebox{\linewidth}{!}{
\begin{tabular}{l>{\columncolor{mybg}}c>{\columncolor{mybg}}ccc}
\toprule
& \multicolumn{2}{c}{\textbf{EFlow} (ours)} & \multicolumn{2}{c}{\textbf{FLM}} \\
\cmidrule(lr){2-3} \cmidrule(lr){4-5}
\textbf{Steps} & Gen. PPL $\downarrow$ & Ent. & Gen. PPL $\downarrow$ & Ent. \\
\midrule
1024 & \textbf{103.63} & 4.16 & 111.36 & 4.30 \\
512  & \textbf{110.97} & 4.16 & 114.89 & 4.31 \\
256  & \textbf{117.69} & 4.16 & 120.19 & 4.32 \\
128  & \textbf{126.20} & 4.15 & 129.06 & 4.34 \\
64   & \textbf{133.62} & 4.14 & 143.73 & 4.36 \\
\bottomrule
\end{tabular}
}
\end{small}
\end{minipage}%
\hfill
\begin{minipage}[t]{0.55\textwidth}
\centering
\caption{\textbf{EFM performance on LM1B.} Generative PPL ($\downarrow$) and entropy ($\uparrow$) across step counts 1, 2, 4. All reported baseline results taken from \citet{lee2026flow}. Best values are \textbf{bolded} and second-best are \underline{underlined}.}
\label{table:efm-language-results}
\begin{small}
\resizebox{\linewidth}{!}{
\begin{tabular}{l cc cc cc}
\toprule
Steps & \multicolumn{2}{c}{\textbf{1}} & \multicolumn{2}{c}{\textbf{2}} & \multicolumn{2}{c}{\textbf{4}} \\
\cmidrule(lr){1-1} \cmidrule(lr){2-3} \cmidrule(lr){4-5} \cmidrule(lr){6-7}
Method & Gen. PPL $\downarrow$ & Ent. & Gen. PPL $\downarrow$ & Ent. & Gen. PPL $\downarrow$ & Ent. \\
\midrule
Duo + DCD   & 1224.52 & 4.33 & 520.08 & 4.20 & 210.88 & 4.23 \\
Duo + Di4C  & 292.94  & 3.79 & 247.69 & 3.87 & 150.67 & 4.00 \\
MDLM + SDTT & 1429.48 & 4.31 & 602.14 & 4.28 & 241.01 & 4.28 \\
MDLM + Di4C & 1217.10 & 4.38 & 621.59 & 4.37 & 247.32 & 4.00 \\
CFM         & 269.72  & 3.10 & 267.39 & 3.15 & 267.97 & 3.28 \\
FMLM        & \underline{119.34}  & 4.16 & \textbf{110.19} & 4.21 & \textbf{98.76}  & 4.21 \\
\midrule
\rowcolor{mybg} \textbf{EFM} (Ours) & \textbf{98.35} & 3.56 & \underline{125.95} & 3.95 & \underline{99.85} & 3.93 \\
\bottomrule
\end{tabular}
}
\end{small}
\end{minipage}%
\end{table}

\paragraph{Results}
As shown in Table \ref{table:language-results}, EFlow improves on the fixed-length flow language model FLM \citep{lee2026flow} at every step budget we evaluate, despite solving the strictly more expressive variable-length problem. Generative perplexity falls monotonically with the sampling budget, from $133.62$ at $64$ steps to $103.63$ at $1024$ steps. Notably, EFlow beats FLM in Gen. PPL at all step counts, indicating that variable-length conditioning does not sacrifice sample quality. Entropy is flat across the sweep at $4.14$-$4.16$, slightly below FLM's $4.30$-$4.36$, but remains coherent as shown in Figure \ref{fig:eflow_lm1b_samples}. Together, these results indicate that the expand-transport decomposition adds variable-length flexibility without sacrificing sample quality. 

We then evaluate the distilled EFM student in the few-step regime (Table \ref{table:efm-language-results}) against distilled diffusion and flow-map baselines: Duo \citep{sahoo2025diffusion} and MDLM \citep{sahoo2024simple} distilled with DCD, SDTT \citep{deschenaux2025beyond}, or Di4C \citep{hayakawa2024distillation}, categorical flow maps (CFM) \citep{roos2026categorical}, and flow map language models (FMLM) \citep{lee2026flow}. EFM generates coherent text at 2 and 4 steps, improving on all distilled diffusion baselines and on CFM by a wide margin, while remaining competitive with FMLM (example generations in Fig \ref{fig:efm_lm1b_samples}). Single-step sampling exhibits mode collapse, which we attribute to the difficulty of predicting the full-length expansion and applying the flow map over the entire interval in one shot (see App \ref{app:one-step} for further discussion). Sample entropy is also lower than the fixed-length baselines, with one likely factor being that EFM conditions on the strictly larger space of local time coordinates rather than a shared uniform schedule, and we expect performance to improve with further schedule tuning. To our knowledge, this is nonetheless the first demonstration of flow maps on variable-length discrete sequences, expanding the design space of existing flow map frameworks that operate on a fixed-length token grid with uniform time conditioning.

\section{Conclusion}

In this work, we introduce \textbf{Expanding Generative Flows} (EFlow) and \textbf{Expanding Flow Maps} (EFMs), a unified framework that recasts generative modeling as learning a transport along an expanding interpolant between marginals of increasing dimensionality. We introduce the \textit{expand operator}, which projects a state into a higher-dimensional space via conditional noise augmentation, and the \textit{transport map}, which pushes the expanded state forward toward the target distribution. By extending continuous- and discrete-space consistency identities to this setting, EFMs enable efficient few-step generation across dimensionalities within a single formulation, removing the limitation of flow-based generators to a fixed-dimensional state space. 

\section{Limitations and Future Work}
The current version of the paper validates EFlow and EFM on a diverse array of modalities, with GEOM-QM9 having a maximum of 29 atoms, GEOM-Drugs a maximum of 181 atoms, and LM1B capped at 128-length sequences. We theoretically show the stability of our Binomial insertion parameterization and transport map on variable-sized (active dimensionality) inputs. However, the complexity introduced by conditioning on local time coordinates for each dimension (over a fixed schedule) and the larger space of interpolants per clean data sample suggests that further tuning is required to scale to larger dimensionalities, which we were unable to fully explore empirically due to the large computational expense. The general construction of several components of our framework also suggests future extensions. For instance, the conditional noise distribution is not restricted to be defined as Gaussian noise and can be an additional learnable parameter, where the mean and covariance from which the inserted dimensions are sampled are parameterized. Furthermore, different interpolants that are \textit{conditioned} on the insertion time, or enforcing some structure to the ordering of insertion times, remain promising directions to be explored. Finally, while we restrict ourselves to interpolants of \textit{increasing dimensionality}, we note that decreasing dimensionality along the interpolant can be considered as an operation in the PDMP construction.

\section*{Declarations}
\paragraph{Acknowledgements} We thank Mark III Systems for providing database and hardware support that has contributed to the research reported within this manuscript. 

\paragraph{Author Contributions} S.T. devised and developed model architectures and theoretical formulations and performed experiments. S.T. drafted the manuscript and designed the figures. P.C. supervised and directed the study, and reviewed and finalized the manuscript.

\paragraph{Data and Materials Availability} The codebase is freely accessible to the academic community at \url{https://github.com/sophtang/ExpandingFlowMaps} and at \url{https://huggingface.co/ChatterjeeLab/ExpandingFlowMaps}.

\paragraph{Funding Statement} This research was supported by NIH grant R35GM155282 to the lab of P.C..

\bibliographystyle{acl_natbib.bst}
\bibliography{citation.bib}

\clearpage
\beginappendix
\startcontents[app]
\printcontents[app]{l}{1}{\setcounter{tocdepth}{2}}

\newpage

\section{Related Works}
\label{app:related-works}
\paragraph{Few-Step Generative Models}
Generation under a flow- or diffusion-based model requires numerical integration of an ODE or SDE, and inference cost scales with the number of function evaluations. A long line of work amortizes this cost by learning a solution operator that maps between two points on the generative trajectory in a single network call. Consistency models \citep{song2023consistency, song2023improved} train networks to be self-consistent along the probability-flow ODE, and flow maps generalize this by transporting samples between any two times $s$ and $t$, recovering consistency models at $s=0$ and the instantaneous velocity as $s\to t$. Recent work develops the theory and training recipes for flow maps \citep{boffi2025build, geng2025mean, frans2024one, guo2025splitmeanflow, sabour2025align}, with parallel efforts on few-step discrete language generation \citep{hu2026elf, deschenaux2025beyond, gloeckle2024better, hayakawa2024distillation, yoo2026redi, roos2026categorical, lee2026flow, potaptchik2026discrete}.

\paragraph{Continuous Flows for Discrete Generation}
Discrete diffusion \citep{shi2024simplified, sahoo2024simple, ou2024your, zheng2024masked} needs many function evaluations for accurate generation, motivating extensions of flow maps to discrete domains where standard regression losses are unstable. Continuous-space parameterizations for discrete generation operate either in simplex space \citep{stark2024dirichlet, davis2024fisher, tang2025gumbel} or in continuous embedding spaces \citep{cheng2025alpha, chen2026langflow, dieleman2022continuous, hu2026elf, deschenaux2026language, yang2026continuous}. Building on these, FMLMs \citep{lee2026flow} and DFMs \citep{potaptchik2026discrete} distill a continuous flow over simplex-embedded tokens into a few-step generator trained with simplex-aware cross-entropy, while CFM \citep{roos2026categorical} self-distills a flow-matching model constrained to the simplex. Together, these show the flow-map recipe transfers to categorical data once the loss matches the simplex geometry.

\paragraph{Stochastic Flow Maps}
Standard flow maps are deterministic, but posterior sampling, reward alignment, and inverse problems require many independent endpoints from a single intermediate state. Flow Map Matching \citep{boffi2024flow} builds flow maps for stochastic interpolants that generalize consistency models. Diamond Maps \citep{holderrieth2026diamond} amortize many SDE steps into one call while preserving the noise injection needed for value estimation, distilled from GLASS Flows \citep{holderrieth2025glass}. Meta Flow Maps \citep{potaptchik2026meta} condition on an intermediate pair $(\boldsymbol{x},t)$ to learn a family of conditional flow maps generating the posterior $p_{u|t}(\cdot\mid\boldsymbol{x})$, recovering the clean posterior $p_{1|t}(\cdot\mid\boldsymbol{x})$ at $u=1$.

\paragraph{Variable-Length Generation}
Most non-autoregressive models, including masked diffusion and standard discrete flows, operate on a fixed-length canvas that forces a length commitment upfront and wastes compute on padding. A growing body of work removes this restriction. FlexMDMs \citep{kim2025any} extend stochastic interpolants \citep{albergo2025stochastic, albergo2022building} to build sequences by inserting and then unmasking tokens, and can retrofit pretrained models like LLaDA-8B. Edit Flows \citep{havasi2025edit} model generation as a continuous-time Markov chain over insertion, deletion, and substitution edits, and OneFlow \citep{nguyen2025oneflow} extends this to interleaved multimodal text-and-image outputs. Within masked diffusion, several methods adjust length via the model's own predictions, including block diffusion \citep{arriola2025block}, DAEDAL \citep{li2025beyond}, $\rho$-EOS \citep{yang2026rho}, DreamOn \citep{wu2026dreamon}, and Deletion-Insertion Diffusion \citep{ding2026beyond}. A2D2 \citep{tang2026a2d2} formulates any-length generation as a controlled continuous-time Markov chain and jointly optimizes insertion and unmasking policies for reward-guided fine-tuning. These works establish variable-length generation but remain outside the flow-map paradigm, motivating how to combine length flexibility with few-step flow-map efficiency.

\section{Theoretical Proofs}

\subsection{Notation}
\paragraph{State spaces and time variables}
We write $s,u,t\in[0,1]$ for time coordinates and $d(t):[0,1]\to\mathbb{N}$ for a non-decreasing dimension schedule. In the continuous setting $\mathbb{R}^{d(t)}$ is the state space at time $t$ and, in the discrete setting $\mathbb{R}^{d(t)\times V}$ is the space of length-$d(t)$ sequences whose tokens are $V$-dimensional latent vectors that resolve onto the simplex $\Delta^{V-1}\subset\mathbb{R}^V$. The PDMP of Section~\ref{sec:expanding-flows} lives on the disjoint union $\mathcal{M}:=\bigsqcup_{t}\mathbb{R}^{d(t)}$. For each coordinate $i$, $t_i^{\text{ins}}$ is its insertion time and $t_i=(t-t_i^{\text{ins}})/(1-t_i^{\text{ins}})$ its local time. The vector of local times is denoted $\boldsymbol{t}_{\text{local}}$.

\paragraph{States and distributions}
$\boldsymbol{x}_t$ denotes the state at time $t$, a latent vector in $\mathbb{R}^{d(t)}$ (continuous) or a sequence of latent vectors in $\mathbb{R}^{d(t)\times V}$ (discrete). We write $\boldsymbol{x}_0\sim p_0$ for a prior sample, $\boldsymbol{x}_1\sim p_1$ for a clean data sample, and $\boldsymbol{x}_1^{(t)}$ for the data sample restricted to the coordinates present at time $t$. The marginal law is $p_t\in\mathcal{P}(\mathbb{R}^{d(t)})$ and the conditional law of the jump is $p_{t|s}(\cdot\mid\boldsymbol{x}_s)$. Augmented noise $\boldsymbol{\epsilon}\sim p_{\epsilon|s,t}$ lifts the dimension from $d(s)$ to $d(t)$, giving the augmented state $\boldsymbol{x}^\epsilon_s:=\mathcal{E}_{s,t}(\boldsymbol{x}_s)=\mathcal{E}_{s,t}(\boldsymbol{x}_s,\boldsymbol{\epsilon})$.

\paragraph{Operators}
$\mathcal{E}_{s,t}$ is the expand operator, $X_{s,t}$ the transport map in the augmented space, and $\Phi_{s,t}:=X_{s,t}\circ\mathcal{E}_{s,t}$ the full expanding flow map. Transport is parameterized by the average velocity $v_{s,t}$ in the continuous setting and by the mean denoiser $\psi_{s,t}$ (valued on the simplex) in the discrete setting. On the diagonal $t=s$, $b_t$ is the instantaneous velocity and $D_t$ the instantaneous denoiser. In the discrete setting, $\alpha_t$ is the insertion schedule with $\Pr[t_i^{\text{ins}}\le t]=\alpha_t$, $\mathcal{I}_{s,t}(\boldsymbol{x}_s)[i]$ the per-gap insertion expectation, and $\lambda_t$ the insertion intensity of the PDMP. Hats denote network-parameterized objects ($\hat{v}_{s,t}$, $\hat{\psi}_{s,t}$, $\hat{\mathcal{E}}_{s,t}$, $\hat{\mathcal{I}}_{s,t}$) and $\text{sg}(\cdot)$ the stop-gradient operator used in the consistency objectives.

\subsection{Assumptions}

\paragraph{Assumptions on the Expand Operator}
The expand operator $\mathcal{E}_{s,t}$ obeys the following structural properties:

\begin{assumption}[Fixed-noise differentiability]\label{ass:A1}
Once the augmented noise $\beps$ is sampled, $\Estt(\bx,\beps)$ depends differentiably on $\bx$, and its $\bx$-derivative is the embedding matrix $E\in\Real^{d(t)\times d(s)}$ that places existing coordinates at their target positions and fills the rest with the entries of $\beps$. In particular, $\partial_{t}\Estt(\bx,\beps)=0$ and $\partial_{s}\Estt(\bx,\beps)=0$ at fixed $\beps$ and fixed placement, and $\nabla_{\bx}\Estt(\bx,\beps)=E$.
\end{assumption}

\begin{assumption}[Composability]\label{ass:A2}
For $s\le u\le t$, the placements compose: $\Estt=\E_{u,t}\circ\E_{s,u}$ when the noise samples are drawn consistently from the disintegration $p_{\epsilon|s,t}=p_{\epsilon| u,t}\otimes p_{\epsilon| s,u}$.
\end{assumption}

\begin{assumption}[Inactive coordinates are unchanged]\label{ass:A3}
For $s\le u$, the augmented-space flow $X_{s,u}$ on $\Real^{d(t)}$ does not change the $d(t)-d(u)$ coordinates that are scheduled to be inserted only at times $\ge u$. Equivalently, $X^{(t)}_{s,u}\circ\E_{u,t}=\E_{u,t}\circ X_{s,u}$ on the image of $\E_{s,u}$.
\end{assumption}

\begin{assumption}[Active-inactive velocity definition]\label{ass:A4}
On the image of $\Estt$, the diagonal velocity has no component along inactive coordinates, such that there exists $\tilde v_{s,s}:\Real^{d(s)}\to\Real^{d(s)}$ such that:
\begin{align}
    v_{s,s}(\Estt(\bx))=E\tilde v_{s,s}(\bx),\quad E:=\nabla_{\bx}\Estt(\bx,\beps)\in\Real^{d(t)\times d(s)}
\end{align}
\end{assumption}

Intuitively, \ref{ass:A3} states that the flow does not transform coordinates before they are inserted. \ref{ass:A4} extends \ref{ass:A3} by stating that if inactive coordinates do not move under the augmented flow, then their instantaneous velocity vanishes, so $v_{s,s}(\Estt(\bx))$ is supported on the active block (i.e., in the column space of $E$). This means that it can be reduced back to a $d(s)$-dimensional velocity $\tilde v_{s,s}:\Real^{d(s)}\to \Real^{d(s)}$.

\subsection{Assumptions Hold In Practice}
\label{app:assumptions-experiments}

The assumptions on the expand operator split into two groups. \ref{ass:A1} (fixed-noise differentiability) and \ref{ass:A2} (composability) constrain only the placement map of $\mathcal{E}_{s,t}$ and hold for any learned velocity, while \ref{ass:A3} (inactive coordinates are unchanged) and \ref{ass:A4} (active-inactive split) constrain the learned dynamics and are enforced by the local-time velocity of (\ref{eq:flow-velocity-field}). We verify both groups hold by construction in each experiment.

\paragraph{\ref{ass:A1} and \ref{ass:A2} are automatically enforced by the expand operation.}
In every experiment, $\mathcal{E}_{s,t}$ only decides which slot each coordinate goes in. It copies the existing coordinates of $\bx_s$ into their slots and drops the noise $\beps$ into the new ones, without transforming any values. Once the noise and the slots are fixed, the output is just $\bx_s$ rearranged plus a constant, so it is affine in $\bx_s$ with derivative equal to the fixed $0/1$ slot-selection matrix $E$ and no dependence on time, giving \ref{ass:A1}. This holds for concatenation, gap-wise insertion, and child expansion. 

Composability (\ref{ass:A2}) then says that inserting in two steps ($s\to u\to t$) lands each coordinate in the same slot as inserting in one step ($s\to t$), which holds because concatenating slots is associative. The only condition is that the augmented noise be drawn consistently across the two paths. For a fixed schedule, each token's insertion time satisfies $\text{Pr}[t_i^{\text{ins}}\leq t]=\alpha_t$, which satisfies the disintegration $p_{\epsilon|s,t}=p_{\epsilon|u,t}\otimes p_{\epsilon|s,u}$. When the per-gap allocation is learned (language), the schedule factor in (\ref{eq:insert-expectation}) and the two-time insertion loss constrain the head to respect this disintegration, so composability is enforced by the fixed schedule.

\paragraph{\ref{ass:A3} and \ref{ass:A4} are enforced by the velocity mask.}
The velocity field sets $b_t(\bx^\epsilon_t)[i]=\boldsymbol{0}$ for every uninserted coordinate ($t^{\text{ins}}_i>t$), so uninserted coordinates never move (\ref{ass:A3}) and $v_{s,s}(\mathcal{E}_{s,t}(\bx))$ is supported on the active block spanned by $E$ (\ref{ass:A4}). This mask is applied at the parameterization level and does not rely on training. In conformer generation, the heavy-atom backbone is denoised and frozen before hydrogens are inserted, which is \ref{ass:A3} applied to resolved atoms. In discrete graphs, the atom count is fixed and uninserted node and edge coordinates stay empty until inserted. In language, a position enters the mean-denoiser loss only after being inserted, so the denoiser emits no gradient for empty positions.

\subsection{Proof of Proposition \ref{prop:pdmp}}
\label{app-prop:pdmp}

\begin{tcolorbox}[sharp corners, colback=mybg, boxrule=0pt]
\pdmp*
\end{tcolorbox}

\textit{Proof.}
A piecewise-deterministic Markov process (PDMP) on $\mathcal{M}$ is characterized by (i) a deterministic flow $\dot \bx_\tau$, (ii) a jump intensity $\lambda_t\geq 0$, and (iii) a transition kernel $Q_t$ of the jump process. The generator of a PDMP is given by:
\begin{align}
    \mathcal{L}_tf=\nabla f\cdot b_t+\lambda_t\int\left(f(\by)-f(\bx)\right)Q_t(\bx, d\by)\label{eq:app-pdmp-generator}
\end{align}
Because the jump intensity and transport depend on how far each coordinate has been denoised, we track the full state $\boldsymbol{S}_t:=(\bx_t, \boldsymbol{t}_{\text{local}})$, where $\boldsymbol{t}_{\text{local}}=\{t_i\}$ collects the local time coordinates $t_i=(t-t_i^{\text{ins}})/(1-t_i^{\text{ins}})$ of the inserted coordinates (equivalently the insertion times, by the bijection at fixed $t$). We write each component of the expanding generative flow as a component of the PDMP generator.

\textit{Step 1: Deterministic Flow. }
Given a ordered set of insertion times $\{t_k^{\text{ins}}\}$, the time interval between two consecutive insertion times $(t^{\text{ins}}_k, t^{\text{ins}}_{k+1})$ has fixed dimension $d(t)$ and following Def \ref{def:expanding-flow}, the state evolves via an augmented-space ODE $\dot \bx_\tau=b_\tau(\bx_\tau)$ on the fixed-dimensional space $\bx_\tau\in \mathbb{R}^{d(t^{\text{ins}+}_{k})}$, where the superscript $+$ denotes the right limit, i.e.\ the dimension immediately after the $k$th insertion. For any test function $f\in C^1_b$ with bounded, continuous first derivative, the chain rule yields:
\begin{align}
    \tfrac{d}{dt}f(\bx_t)=\dot \bx_t\cdot \nabla f(\bx_t)=b_t(\bx_t)\cdot \nabla f(\bx_t)
\end{align}
which is the transport term in (\ref{eq:app-pdmp-generator}). By \ref{ass:A3}, $b_\tau(\bx)[i]=0$ for every uninserted coordinate $i$ and all $\bx$, so the velocity field maps into the active subspace and its flow keeps the uninserted coordinates fixed for the whole interval. Hence $b_t\cdot \nabla f$ is supported on the inserted coordinates and the flow stays bound to $\mathbb{R}^{d(t^{\text{ins}+}_{k})}$.

\textit{Step 2: Jump Between Dimensions. }
At each insertion time $t^{\text{ins}}_k$, the state makes a discontinuous jump from $\bx\in \mathbb{R}^{d(t^{\text{ins}}_{k-})}$ to the augmented state $\bx^\epsilon=\mathcal{E}_{t^{\text{ins}}_{k}}(\bx, \boldsymbol{\epsilon})\in \mathbb{R}^{d(t^{\text{ins}+}_{k})}$ where $\boldsymbol{\epsilon}\sim p_{\eps|t^{\text{ins}}_k}$. Since the expanding operator $\mathcal{E}_{t^{\text{ins}}_{k}}(\bx, \boldsymbol{\epsilon})$ depends only on the current state $\bx$ and freshly sampled noise $\boldsymbol{\epsilon}$, the jump is Markov. 

By \ref{ass:A1}, for fixed noise $\mathcal{E}_t(\cdot,\boldsymbol{\epsilon})$ is affine in both $\bx$ and  $\boldsymbol{\epsilon}$, so as a map of the pair $(\bx,\boldsymbol{\epsilon})$ it is continuous and Borel measurable. This guarantees that the preimage $\{\boldsymbol{\epsilon}:\mathcal{E}_t(\bx, \boldsymbol{\epsilon})\in A\}$ is measurable and the transition kernel $Q_t(\bx, A)$, defined as the probability of the expanded state starting from $\bx$ landing in a measurable set $A$, is well-defined. Given $\bx$, the stochasticity is dependent on $\boldsymbol{\epsilon}$, so the probability of landing in $A$ is the probability that $\boldsymbol{\epsilon}$ falls in the preimage set, equal to the pushforward measure:
\begin{align}
    Q_t(\bx, A)=p_{\epsilon|t}(\{\boldsymbol{\epsilon}:\mathcal{E}_t(\bx, \boldsymbol{\epsilon})\in A\})=\left(\left(\mathcal{E}_t(\bx, \cdot)\right)_{\#}p_{\epsilon|t}\right)(A)
\end{align}
so $Q_t(\bx,\cdot)$ is the law of $\mathcal{E}_t(\bx,\boldsymbol{\epsilon})$, a Markov kernel. The generator of a Markov jump process is the expected instantaneous change in the test function with jump intensity $\lambda_t$: 
\begin{align}
    \lambda_t\int\left(f(\by)-f(\bx)\right)Q_t(\bx, d\by)=\lambda_t\int\left(f(\mathcal{E}_t(\bx, \boldsymbol{\epsilon})\right)-f(\bx))p_{\epsilon|t}(d\boldsymbol{\epsilon})
\end{align}
where we substitute the post-jump state with the expanded state $\mathcal{E}_t(\bx, \boldsymbol{\epsilon})$ and integrate over $\boldsymbol{\epsilon}\sim p_{\epsilon|t}$.

\textit{Step 3: Markov Property. }
The post-jump state depends only on $(\bx, \boldsymbol{\epsilon})$ with $\boldsymbol{\epsilon}$ drawn independently of the past. Between jumps the local times advance deterministically by $\dot t_i=(1-t_i^{\text{ins}})^{-1}$ and a newly inserted coordinate enters at $t_i=0$, so $\boldsymbol{t}_{\text{local}}$ is determined by the state. 

Since between jumps $\boldsymbol{S}_t$ evolves deterministically and at a jump depends only on $(\bx,\boldsymbol{\epsilon})$, it depends on the past only through $\boldsymbol{S}_{t^-}$, so $(\boldsymbol{S}_t)_{t\in [0,1]}$ is Markov. The state process $(\bx_t)_{t\in [0,1]}$ is not Markov on its own, as its intensity and transport require $\boldsymbol{t}_{\text{local}}$. By Prop \ref{prop:pushforward}, the conditional law is $p_{t|s}(\cdot\mid\bx_s)=(\mathcal{E}_{s,t}(\bx_s,\cdot))_{\#}p_{\epsilon|s,t}$, and taking $s\uparrow t$ and comparing with the kernel above gives $Q_t(\boldsymbol{x}_s,\cdot)=p_{t|s}(\cdot\mid\boldsymbol{x}_s)$, where the disintegration $p_{\epsilon|s,t}=p_{\epsilon|u,t}\otimes p_{\epsilon|s,u}$ is \ref{ass:A2}.

\textit{Step 4: Jump Intensity in Discrete Case. }
A coordinate inserted with schedule $\Pr[t^{\text{ins}}_i\le t]=\alpha_t$ has survival $1-\alpha_t$ and hazard $-\tfrac{d}{dt}\log(1-\alpha_t)=\dot\alpha_t/(1-\alpha_t)$ (the insertion rate for each expected uninserted token) so summing over gaps $i$ gives $\lambda_t(\bx)=\tfrac{\dot\alpha_t}{1-\alpha_t}\sum_i\mathcal{I}_t(\bx)[i]$. Since each jump raises the dimension by at least one and $d(1)<\infty$, there are finitely many jumps on $[0,1]$ and the process does not explode. Combining the generator components from the deterministic flow, insertion jump, and jump intensity yields the generator for the PDMP of the expanding flow in (\ref{eq:pdmp-generator}). \hfill $\square$

\subsection{Proof of Proposition \ref{prop:pushforward}}
\label{app-prop:pushforward}

\begin{tcolorbox}[sharp corners, colback=mybg, boxrule=0pt]
\pushforward*
\end{tcolorbox}

\textit{Proof.}
Fix $\boldsymbol{x}_s\in\mathbb{R}^{d(s)}$ and hold it constant throughout, so that $\mathcal{E}_{s,t}(\boldsymbol{x}_s,\cdot):\mathbb{R}^{d(t)-d(s)}\to\mathbb{R}^{d(t)}$ is a measurable map in the augmented noise alone (measurability holds for each instantiation, since concatenation, positional insertion, and coordinate refinement are all compositions of coordinate embeddings and permutations, which are Borel measurable). The augmented state is the random variable:
\begin{align}
    \boldsymbol{x}_t := \mathcal{E}_{s,t}(\boldsymbol{x}_s,\boldsymbol{\epsilon}), \quad \boldsymbol{\epsilon}\sim p_{\epsilon\mid s,t}(\cdot\mid\boldsymbol{x}_s),
\end{align}
and by definition the conditional law $p_{t\mid s}(\cdot\mid\boldsymbol{x}_s)$ is the distribution of $\boldsymbol{x}_t$ given $\boldsymbol{x}_s$, i.e. the law of $\mathcal{E}_{s,t}(\boldsymbol{x}_s,\boldsymbol{\epsilon})$ under the noise law.

By the definition of the pushforward measure, for every Borel set $A\subseteq\mathbb{R}^{d(t)}$,
\begin{align}
    p_{t\mid s}(A\mid\boldsymbol{x}_s)
    &= \Pr\big[\mathcal{E}_{s,t}(\boldsymbol{x}_s,\boldsymbol{\epsilon})\in A \;\big|\; \boldsymbol{x}_s\big] \\
    &= \Pr\big[\boldsymbol{\epsilon}\in \mathcal{E}_{s,t}(\boldsymbol{x}_s,\cdot)^{-1}(A) \;\big|\; \boldsymbol{x}_s\big] \\
    &= p_{\epsilon\mid s,t}\big(\mathcal{E}_{s,t}(\boldsymbol{x}_s,\cdot)^{-1}(A)\;\big|\;\boldsymbol{x}_s\big) \\
    &= \big(\mathcal{E}_{s,t}(\boldsymbol{x}_s,\cdot)\big)_{\#} p_{\epsilon\mid s,t}(\cdot\mid\boldsymbol{x}_s) (A),
\end{align}
where the third line uses that the preimage $\mathcal{E}_{s,t}(\boldsymbol{x}_s,\cdot)^{-1}(A)$ is measurable, and the last line is exactly the definition of the pushforward of $p_{\epsilon\mid s,t}(\cdot\mid\boldsymbol{x}_s)$ under the map $\mathcal{E}_{s,t}(\boldsymbol{x}_s,\cdot)$. Since this holds for all Borel $A$, the two measures agree, giving (\ref{eq:expanding-cond-law}).\hfill $\square$

\subsection{Proof of Proposition \ref{prop:expanding-fm-identities}}
\label{app-prop:expanding-fm-identities}
\paragraph{Setup}
Let $\boldsymbol{x}^\epsilon_s=\mathcal{E}_{s,t}(\bx)\in\mathbb{R}^{d(t)}$ and $\Xst:\mathbb{R}^{d(t)}\to\mathbb{R}^{d(t)}$ be the time-$t$ flow map of the augmented-space ODE $\dot{\boldsymbol{z}}_{\tau}=b_{\tau}(\boldsymbol{z}_{\tau})$ on $\mathbb{R}^{d(t)}$ with $\boldsymbol{z}_{s}=\boldsymbol{x}_s^\epsilon$. By the proofs in \citet{boffi2025build}, $\Xst$ satisfies the standard fixed-dimensional flow map identities. Here, we lift these through $\mathcal{E}_{s,t}$ to $\Phi_{s,t}=X_{s,t}\circ\mathcal{E}_{s,t}$.

\begin{tcolorbox}[sharp corners, colback=mybg, boxrule=0pt]
\expandingfm*
\end{tcolorbox}

\textit{Proof.} Throughout this proof, we fix the augmented noise $\beps$ and simplify notation by writing $\Estt(\bx)\equiv\Estt(\bx,\beps)$. By \ref{ass:A1} the placement is held fixed, so the Lagrangian and Eulerian identities are local statements valid on any interval of times over which the dimension schedule $d(\cdot)$ is constant, i.e., between consecutive insertion times. The semigroup identity requires no differentiability and holds for all $0\le s\le u\le t\le 1$. We prove each identity one by one.

\paragraph{Lagrangian Identity}
By the chain rule applied to $\Phisst(\bx)=\Xst(\Estt(\bx))$, we have:
\begin{align}
    \partial_{t}\Phisst(\bx)
    =\partial_{t}\Xst(\by)\big|_{\by=\Estt(\bx)}
    +\nabla_{\by}\Xst(\by)\big|_{\by=\Estt(\bx)}\cdot\underbrace{\partial_{t}\Estt(\bx,\beps)}_{=0\;\text{by \ref{ass:A1}}}
    =\partial_{t}\Xst(\by)\big|_{\by=\Estt(\bx)}
\end{align}
The standard Lagrangian condition (\ref{eq:consistency-constraint}) in the fixed-dimensional augmented  space $\Real^{d(t)}$, $\partial_{t}\Xst(\by)=v_{t,t}(\Xst(\by))$, evaluated at $\by=\Estt(\bx)$, gives:
\begin{align}
    \partial_{t}\Phisst(\bx)=v_{t,t}\big(\Xst(\Estt(\bx))\big)=v_{t,t}(\Phisst(\bx)),
\end{align}
which is the Lagrangian identity in (\ref{eq:expanding-fm-identities}).

\paragraph{Eulerian Identity}
By \ref{ass:A1}, $\partial_{s}\Estt(\bx,\beps)=0$ at fixed $\beps$, so the chain rule in $s$ gives:
\begin{align}\label{eq:chain-s}
    \partial_{s}\Phisst(\bx)&=\partial_{s}\Xst(\by)\big|_{\by=\Estt(\bx)}
    +\nabla_{\by}\Xst(\by)\big|_{\by=\Estt(\bx)}\cdot\underbrace{\partial_{s}\Estt(\bx,\beps)}_{=0}\nonumber\\
    &=\partial_{s}\Xst(\by)\big|_{\by=\Estt(\bx)}
\end{align}
Similarly, since $\nabla_{\bx}\Estt(\bx,\beps)=E\in\Real^{d(t)\times d(s)}$ by \ref{ass:A1}, the chain rule in $\bx$ gives:
\begin{align}\label{eq:chain-x}
    \nabla_{\bx}\Phisst(\bx)=\nabla_{\by}\Xst(\by)\big|_{\by=\Estt(\bx)} E\;\in\;\Real^{d(t)\times d(s)}
\end{align}
The standard Eulerian condition (\ref{eq:consistency-constraint}) on the augmented space $\Real^{d(t)}$, evaluated at $\by=\Estt(\bx)$, is given by:
\begin{align}\label{eq:euler-aug}
    \partial_{s}\Xst(\by)+\nabla_{\by}\Xst(\by) v_{s,s}(\by)=0
\end{align}
By \ref{ass:A4}, on the image of $\Estt$ the diagonal velocity satisfies $v_{s,s}(\Estt(\bx))=E\tilde v_{s,s}(\bx)$ with $\tilde v_{s,s}(\bx)\in\Real^{d(s)}$. Substituting into the second term of (\ref{eq:euler-aug}) and using (\ref{eq:chain-x}), we get:
\begin{align}
    \nabla_{\by}\Xst(\by)\big|_{\by=\Estt(\bx)}v_{s,s}(\Estt(\bx))
    =\nabla_{\by}\Xst(\by)\big|_{\by=\Estt(\bx)}E \tilde v_{s,s}(\bx)
    \stackrel{(\ref{eq:chain-x})}{=}\nabla_{\bx}\Phisst(\bx) \tilde v_{s,s}(\bx)
\end{align}
Combining with (\ref{eq:chain-s}) yields:
\begin{align}
    \partial_{s}\Phisst(\bx)+\nabla_{\bx}\Phisst(\bx) \tilde v_{s,s}(\bx)=0,
\end{align}
which is the Eulerian identity in (\ref{eq:expanding-fm-identities}). Both terms are vectors in $\Real^{d(t)}$, since $\nabla_{\bx}\Phisst(\bx)\in\Real^{d(t)\times d(s)}$ multiplies the vector $\tilde v_{s,s}(\bx)\in\Real^{d(s)}$ to give a vector in $\Real^{d(t)}$, matching $\partial \Phisst(\bx)$.

\paragraph{Semigroup Identity}
Fix $0\le s\le u\le t\le 1$. The maps involved act between spaces of different dimension:
\begin{align}
    \Phi_{s,t}:\Real^{d(s)}\to \Real^{d(t)},\quad 
    \Phi_{s,u}:\Real^{d(s)}\to \Real^{d(u)},\quad 
    \Phi_{u,t}:\Real^{d(u)}\to \Real^{d(t)},
\end{align}
so the claim is that expanding and transporting from $d(s)$ to $d(t)$ in one step agrees with doing so in two steps through $d(u)$. 

Let $X^{(t)}_{s,u}:\Real^{d(t)}\to\Real^{d(t)}$ denote the transport map on the augmented space $\Real^{d(t)}$ over $[s,u]$. By (\ref{eq:flow-velocity-field}) its velocity vanishes on coordinates scheduled for insertion at times $\ge u$, so $X^{(t)}_{s,u}$ acts as $X_{s,u}$ on the coordinates active by time $u$ and as the identity on the remaining $d(t)-d(u)$ coordinates. Since $X^{(t)}_{s,u}$, $X_{u,t}$, and $X_{s,t}$ are all flow maps of the same ODE on the fixed space $\Real^{d(t)}$, the standard semigroup condition (\ref{eq:consistency-constraint}) gives
\begin{align}\label{eq:semi-2}
    \Phi_{s,t}(\bx)=X_{s,t}(\Estt(\bx))=X_{u,t}\big(X^{(t)}_{s,u}(\Estt(\bx))\big).
\end{align}
By \ref{ass:A2}, the placements compose, $\Estt=\E_{u,t}\circ\E_{s,u}$, and by \ref{ass:A3} the augmented-space flow leaves the coordinates inserted at times $\ge u$ unchanged, so on the image of $\E_{s,u}$, we have: 
\begin{align}\label{eq:semi-1}
    X^{(t)}_{s,u}\circ\E_{u,t}=\E_{u,t}\circ X_{s,u}
\end{align}
Substituting \ref{ass:A2} and then (\ref{eq:semi-1}) into (\ref{eq:semi-2}), we get:
\begin{align}
    \Phi_{s,t}(\bx)
    &=X_{u,t}\Big(X^{(t)}_{s,u}\big(\E_{u,t}(\E_{s,u}(\bx))\big)\Big)
    =X_{u,t}\Big(\E_{u,t}\big(X_{s,u}(\E_{s,u}(\bx))\big)\Big)\nonumber\\
    &=X_{u,t}\big(\E_{u,t}(\Phi_{s,u}(\bx))\big)
    =\Phi_{u,t}(\Phi_{s,u}(\bx))
\end{align}
which is the semigroup identity in (\ref{eq:expanding-fm-identities}).\hfill $\square$

\subsection{Proof of Proposition \ref{prop:stochasticfm}}
\label{app-prop:stochasticfm}
\begin{tcolorbox}[sharp corners, colback=mybg, boxrule=0pt]
\stochasticfm*
\end{tcolorbox}

\textit{Proof.} By Definition \ref{def:expanding-flow}, the flow map on the augmented space $X_{s,t}:\mathbb{R}^{d(s)}\times \mathbb{R}^{d(t)-d(s)}\to \mathbb{R}^{d(t)}$ satisfies the pushfoward identity:
\begin{align}
    (X_{s,t})_{\#}p^\epsilon_{s,t}=(X_{s,t})_{\#}(p_{s}\otimes p_{\epsilon|s,t})=p_t
\end{align}
Using the disintegration theorem \citep{kallenberg1997foundations}, we can disintegrate both sides by the marginal $p_s$. The LHS disintegrates by the property that, for fixed $\boldsymbol{x}_s$, the map $X_{s,t}(\boldsymbol{x}_s,\cdot):\Real^{d(t)-d(s)}\to \mathbb{R}^{d(t)}$ is a function of $\beps$ alone, so the conditional measure given $\boldsymbol{x}_s$ is the pushforward of $\peps$ along $\Xst(\bx_{s},\cdot)$:
\begin{align}
    [(X_{s,t})_{\#}p^\eps_{s,t}](\cdot|\bx_s)=(\Xst(\bx_s,\cdot))_\#p_{\eps|s,t}
\end{align}
The RHS of the disintegration written through the joint $p_{s,t}(\bx_s,\bx_t )$, is given by the conditional distribution $p_{t|s}(\cdot| \bx_{s})$:
\begin{align}
    p_{t}(\cdot)=\int p_{t| s}(\cdot\mid x_{s})p_{s}(dx_{s})
\end{align}
By the uniqueness of disintegration, we have:
\begin{align}
    p_{t|s}(\cdot| \bx_{s})=\big(\Xst(\bx_{s},\cdot)\big)_{\#}\peps
\end{align}
The clean posterior is recovered when $t=1$:
\begin{align}
    p_{1| s}(\cdot| \bx_{s})=\big(X_{s,1}(\bx_{s},\cdot)\big)_{\#}p_{\eps| s,1}
\end{align}
which concludes the proof. \hfill $\square$

\subsection{Proof of Proposition \ref{prop:discrete-efm}}
\label{app-prop:discrete-efm}
In the discrete expanding flow map setting, we define the parameterization:
\begin{align}
    \Phisst(\bx_{s})=\Xst(\Estt(\bx_{s})), \quad \Xst(\by)=\tfrac{1-t}{1-s}\by+\tfrac{t-s}{1-s}\psist(\by)
\end{align}
which satisfies the tangent condition $\psi_{s,s}(\by)=\by+(1-s)b_{s}(\by)$. We use (\ref{ass:A1}-\ref{ass:A4}) from Proposition \ref{app-prop:expanding-fm-identities} and the additional assumption for the mean denoiser:

\begin{assumption}[Diagonal-expand consistency]\label{ass:A5}
$\psi^{(t)}_{s,s}(\Estt(\bx))=\Estt(\psi_{s,s}(\bx))$ for the diagonal denoiser, where $\psi^{(t)}_{s,s}:\Real^{d(t)\times V}\to(\Delta^{V-1})^{d(t)}$ is the diagonal denoiser at the augmented dimension $d(t)$ and $\psi_{s,s}:\Real^{d(s)\times V}\to(\Delta^{V-1})^{d(s)}$ at the native dimension $d(s)$, so both sides lie in $(\Delta^{V-1})^{d(t)}$. Equivalently, the model's clean prediction at the \textit{diagonal} on a newly lifted state equals the lift of its clean prediction on the unlifted state. 
\end{assumption}

\begin{tcolorbox}[sharp corners, colback=mybg, boxrule=0pt]
\discreteefm*
\end{tcolorbox}
\textit{Proof.} Throughout this proof, we denote $\by:=\xseps=\Estt(\bx_{s})$. 

\paragraph{Lagrangian Identity} 
Differentiating $\Xst(\by)=\tfrac{1-t}{1-s}\by+\tfrac{t-s}{1-s}\psist(\by)$ in $t$ at fixed $s,\by$, we get:
\begin{align}
    \partial_{t}\Xst(\by)=\frac{\psist(\by)-\by}{1-s}+\frac{t-s}{1-s}\partial_{t}\psist(\by)\label{eq:dXdt}
\end{align}
The continuous Lagrangian on the augmented space gives $\partial_{t}\Xst(\by)=b_{t}(\Xst(\by))$, and the boundary condition $\psi_{t,t}(\bz)=\bz+(1-t)b_{t}(\bz)$ implies:
\begin{align}
    b_{t}(\bz)=\frac{\psi_{t,t}(\bz)-\bz}{1-t}
\end{align}
Setting $\bz=\Xst(\by)$ in $b_{t}$ and setting the expression equal to (\ref{eq:dXdt}) multiplied by $(1-t)$, we get:
\begin{align}
    \frac{1-t}{1-s}\big(\psist(\by)-\by\big)+\frac{(1-t)(t-s)}{1-s}\partial_{t}\psist(\by)=\psi_{t,t}(\Xst(\by))-\Xst(\by)
\end{align}
Substituting $\Xst(\by)=\tfrac{1-t}{1-s}\by+\tfrac{t-s}{1-s}\psist(\by)$ on the right-hand side, the $y$-terms cancel and we get:
\begin{align}
    \Big(\tfrac{1-t}{1-s}+\tfrac{t-s}{1-s}\Big)\psist(\by)+\tfrac{(1-t)(t-s)}{1-s}\partial_{t}\psist(\by)=\psi_{t,t}(\Xst(\by))
\end{align}
The terms inside the bracket sum to $1$, so we have:
\begin{align}
    \psist(\by)=\psi_{t,t}(\Xst(\by))-\frac{(1-t)(t-s)}{1-s}\partial_{t}\psist(\by)
\end{align}
Substituting $\by=\Estt(\bx_{s})=\xseps$ and $\Xst(\by)=\Phisst(\bx_{s})$, we get the final identity: 
\begin{align}
    \psist(\xseps)=\psi_{t,t}(\Phisst(\bx_{s}))-\gamma_{s,t} \partial_{t}\psist(\xseps),\quad \gamma_{s,t}=\tfrac{(1-t)(t-s)}{1-s}
\end{align}
which is the Lagrangian identity of discrete expanding flow maps in (\ref{eq:discrete-consistency}).

\paragraph{Eulerian Identity}
Using the partial derivatives:
\begin{align}
    \partial_{s}\left(\tfrac{1-t}{1-s}\right)=\tfrac{1-t}{(1-s)^{2}},\quad \partial_{s}\left(\tfrac{t-s}{1-s}\right)=-\tfrac{1-t}{(1-s)^{2}}
\end{align}
and differentiating $\Xst(\by)=\tfrac{1-t}{1-s}\by+\tfrac{t-s}{1-s}\psist(\by)$ in $s$ at fixed $\by,t$, we have:
\begin{align}
    \partial_{s}\Xst(\by)=-\tfrac{1-t}{(1-s)^{2}}\big(\psist(\by)-\by\big)+\tfrac{t-s}{1-s}\partial_{s}\psist(\by)\label{eq:euler-1}
\end{align}
The continuous Eulerian identity on $\Real^{d(t)}$ states:
\begin{align}
    \partial_{s}\Xst(\by)+J_{\by}\Xst(\by)b_{s}(\by)=0,\quad J_{\by}\Xst(\by)=\tfrac{1-t}{1-s}\boldsymbol{I}+\tfrac{t-s}{1-s}J_{\by}\psist(\by)\label{eq:euler-2}
\end{align}
Combining (\ref{eq:euler-1}) and (\ref{eq:euler-2}), we get:
\begin{align}
    0=-\tfrac{1-t}{(1-s)^{2}}\big(\psist(\by)-\by\big)+\tfrac{t-s}{1-s}\partial_{s}\psist(\by)+\tfrac{1-t}{1-s}b_{s}(\by)+\tfrac{t-s}{1-s}J_{\by}\psist(\by)b_{s}(\by)
\end{align}
Using $b_{s}(\by)=(\psi_{s,s}(\by)-\by)/(1-s)$ in the third term, we get:
\begin{align}
    0=-\tfrac{1-t}{(1-s)^{2}}\big(\psist(\by)-\psi_{s,s}(\by)\big)+\tfrac{t-s}{1-s}\Big[\partial_{s}\psist(\by)+J_{\by}\psist(\by)b_{s}(\by)\Big]
\end{align}
Rearranging and defining $\kappa_{s,t}:=\tfrac{1-t}{(1-s)(t-s)}$, we get:
\begin{align}
    \partial_{s}\psist(\by)+J_{\by}\psist(\by)b_{s}(\by)&=\frac{1-t}{(1-s)(t-s)}\big(\psist(\by)-\psi_{s,s}(\by)\big)\nonumber\\
    &=\kappa_{s,t}\big(\psist(\by)-\psi_{s,s}(\by)\big)
\end{align}
Substituting $\by=\xseps$ and using \ref{ass:A5} to rewrite $\psi^{(t)}_{s,s}(\xseps)=\Estt(\psi_{s,s}(\bx_{s}))$, we have:
\begin{align}
    \partial_{s}\psist(\xseps)+J_{x}\psist(\xseps)b_{s}(\xseps)=\kappa_{s,t}\left(\psist(\xseps)-\Estt(\psi_{s,s}(\bx_{s}))\right)
\end{align}
which recovers the Eulerian identity of discrete expanding flow maps in (\ref{eq:discrete-consistency}).

\paragraph{Semigroup Identity}
For $s\le u\le t$, we introduce $\psi^{(t)}_{s,u}:\Real^{d(t)\times V}\to (\Delta^{V-1})^{d(t)}$ as the two-time mean denoiser over $[s,u]$ applied at the augmented dimension $d(t)$ that leaves the dimensions inserted over $(u,t]$ unchanged, and $\psi_{s,u}:\Real^{d(u)\times V}\to (\Delta^{V-1})^{d(u)}$ for the same denoiser at its native dimension $d(u)$. Since the network is applied to whatever sequence it receives, these are the same parameterized map evaluated on inputs of different length. Given this, we denote the transport map defined by $\psi^{(t)}_{s,u}$ as $X_{s,u}^{(t)}:\Real^{d(t)\times V}\to\Real^{d(t)\times V}$ which satisfies the semigroup condition in augmented space $X_{s,t}=X_{u,t}\circ X_{s,u}^{(t)}$.

For any $s\le u\le t$, we can expand $\psist$ from the definition $\Xst(\by)=\tfrac{1-t}{1-s}\by+\tfrac{t-s}{1-s}\psist(\by)$:
\begin{align}
    \psist(\by)=\frac{1-s}{t-s}\Xst(\by)-\frac{1-t}{t-s}\by\label{eq:disc-semi-2}
\end{align}
Applying the augmented-space semigroup $\Xst=X_{u,t}\circ X^{(t)}_{s,u}$ and the definition of $X_{u,t}$, we get:
\begin{align}
    \Xst(\by)=X_{u,t}(X^{(t)}_{s,u}(\by))=\frac{1-t}{1-u}X^{(t)}_{s,u}(\by)+\frac{t-u}{1-u}\psi_{u,t}(X^{(t)}_{s,u}(\by))\label{eq:disc-semi-1}
\end{align}
Substituting (\ref{eq:disc-semi-1}) into (\ref{eq:disc-semi-2}), we get:
\begin{align}
    \psist(\by)=\frac{(1-s)(1-t)}{(t-s)(1-u)}X^{(t)}_{s,u}(\by)+\frac{(1-s)(t-u)}{(t-s)(1-u)}\psi_{u,t}(X^{(t)}_{s,u}(\by))-\frac{1-t}{t-s}\by
\end{align}
Now expanding $X^{(t)}_{s,u}(\by)=\tfrac{1-u}{1-s}\by+\tfrac{u-s}{1-s}\psi^{(t)}_{s,u}(\by)$ in the first term, we have: 
\begin{align}
    \frac{(1-s)(1-t)}{(t-s)(1-u)}X^{(t)}_{s,u}(\by)=\frac{1-t}{t-s}\by+\frac{(1-t)(u-s)}{(t-s)(1-u)}\psi^{(t)}_{s,u}(\by)
\end{align}
Subtracting out $\tfrac{1-t}{t-s}\by$ cancels the $\by$-term, and we get:
\begin{align}
    \psist(\by)=\frac{(u-s)(1-t)}{(t-s)(1-u)}\psi^{(t)}_{s,u}(\by)+\frac{(t-u)(1-s)}{(t-s)(1-u)}\psi_{u,t}(X^{(t)}_{s,u}(\by))\label{eq:semigroup-prelift}
\end{align}
Defining $\omega_{s,u,t}:=\tfrac{(u-s)(1-t)}{(t-s)(1-u)}$ and verifying:
\begin{align}
    1-\omega_{s,u,t}=\frac{(t-s)(1-u)-(u-s)(1-t)}{(t-s)(1-u)}=\frac{(t-u)(1-s)}{(t-s)(1-u)}
\end{align}
given $(t-s)(1-u)-(u-s)(1-t)=t-tu-s+su-u+ut+s-st=(t-u)(1-s)$, which means that (\ref{eq:semigroup-prelift}) is a convex combination. Substituting $\by=\xseps=\mathcal{E}_{s,t}(\bx_s)$, the second term simplifies by \ref{ass:A2} and \ref{ass:A3}:
\begin{align}
    X^{(t)}_{s,u}(\xseps)=X^{(t)}_{s,u}(\E_{u,t}(\E_{s,u}(\bx_{s})))=\E_{u,t}(X_{s,u}(\E_{s,u}(\bx_{s})))=\E_{u,t}(\Phi_{s,u}(\bx_{s}))
\end{align}
so that
\begin{align}
    \psist(\xseps)=\omega_{s,u,t}\psi^{(t)}_{s,u}(\xseps)+(1-\omega_{s,u,t})\psi_{u,t}\big(\E_{u,t}(\Phi_{s,u}(\bx_{s}))\big)
\end{align}
which recovers the semigroup identity for the discrete expanding flow map in (\ref{eq:discrete-consistency}). Equivalently, keeping the pre-substitution term of (\ref{eq:semigroup-prelift}) at $\by=\xseps$ and writing $\Phi^{(t)}_{s,u}(\xseps):=\tfrac{1-u}{1-s}\xseps+\tfrac{u-s}{1-s}\psi^{(t)}_{s,u}(\xseps)=X_{s,u}(\xseps)$ gives the \textit{expand-once} form:
\begin{align}
    \psist(\xseps)=\omega_{s,u,t}\psi^{(t)}_{s,u}(\xseps)+(1-\omega_{s,u,t})\psi_{u,t}\big(\Phi^{(t)}_{s,u}(\xseps)\big),\label{eq:semigroup-expandonce}
\end{align}
in which both terms act on the single $d(t)$ state $\xseps$. This is the form implemented in Algorithm \ref{alg:training-efm} and used as the semigroup training target (\ref{eq:semigroup-target}). It equals the native form above by \ref{ass:A2}-\ref{ass:A3}.\hfill $\square$

\section{Discrete Sequence Implementation Details}
\label{app:discrete-efm-implementation}

\subsection{Local Time Insertions and Denoising}
To integrate insertion and denoising into a single flow map framework, we must define insertion times for each position along the sequence. These insertions should be determined by the surrounding context and the sequence's denoising state. 

Each position $i$ has insertion time $t^{\text{ins}}_i$ drawn from the insertion schedule $t^{\text{ins}}_i\sim \dot{\alpha}_sds$ and the \textit{local time} starts and integrates over $t_i\in [0,1]$ starting at global time $t=t^{\text{ins}}_i$ where the position is pure noise initialized as $\boldsymbol{x}^i_{t^{\text{ins}}_i}=\boldsymbol{z}$. At global time $t> t^{\text{ins}}_i$, the local time is defined as:
\begin{align}
    t_i=\frac{t-t^{\text{ins}}_i}{1-t^{\text{ins}}_i}\in [0,1]
\end{align}
such that each position is denoised for $t_i=t-t^{\text{ins}}_i$ at time $t$. Concretely, we have:
\begin{align}
    \boldsymbol{x}^i_{t_i}=(1-t_i)\boldsymbol{x}^i_0+t_i\boldsymbol{x}_1^i, \quad \boldsymbol{x}_0^i=\boldsymbol{z}\sim p_0
\end{align}
Therefore, at time $t$, the input to the denoising model is a padded sequence of length $d(t)$, where each active position has a different noise level $\boldsymbol{t}_{\text{local}}=(t_1, \dots, t_{d(t)}, 0, \dots, 0)$, providing rich context for the denoiser.

In addition to the linear per-position interpolant, one can define a per-position coefficient $\beta:[0,1]\to[0,1]$ (with $\beta(0)=0$, $\beta(1)=1$) that remaps the interpolant, independently of the $t^{\text{ins}}\to t$ time reparameterization, to:
\begin{align}
    \boldsymbol{x}^i_{t_i}=(1-\beta(t_i))\boldsymbol{z}+\beta(t_i)\boldsymbol{x}_1^i, \quad \sigma_i=1-\beta(t_i),
\end{align}
recovering the linear case at $\beta(t_i)=t_i$. An \textit{insertion cutoff time} $t_{\text{ins-end}}\in(0,1]$ confines all insertions to the interval $[0,t_{\text{ins-end}}]$ by reparameterizing the schedule as $\alpha_{t_{\text{ins-end}}}(t)=\alpha(\min(t/t_{\text{ins-end}},1))$, so no position is inserted after $t=t_{\text{ins-end}}$ and the interval $[t_{\text{ins-end}},1]$ becomes a pure-denoising phase. Setting $t_{\text{ins-end}}=1$ allows insertions over the full interval. 

\subsection{Time Reparameterization}
We adopt the time-reparameterization scheme of~\citet{lee2026flow}, which shows that for large alphabets the decoding error $P_e(t)$ remains near its initial value across most of $[0,1]$ and collapses only in a thin band near $t=1$. Uniform training times and evenly spaced inference grids therefore spend nearly all their budget where the model has little to learn, worsening as $|V|$ grows. The fix is to compose the time axis with a smooth, strictly increasing $\tau:[0,1]\to[0,1]$ with $\tau(0)=0$, $\tau(1)=1$, drawing $\tau\sim\mathcal{U}[0,1]$ and setting $t=t(\tau)$ during training, and sampling on the warped grid $t_n=t(n/N)$. Requiring equal increments of $\tau$ to correspond to equal reductions in decoding error gives
\begin{align}
    \tau(t) = \frac{P_e(0)-P_e(t)}{P_e(0)} = 1 - \frac{|V|}{|V|-1}P_e(t)
\label{eq:tau-app}
\end{align}
concentrating both the training distribution and the inference grid on the window where token identities are committed. We evaluate \eqref{eq:tau-app} once by Gauss-Hermite quadrature on a uniform grid of $10^4$ points and fit cubic splines in both directions for constant-time evaluation of $\tau(t)$, $t(\tau)$, and $d\tau/dt$. The warp is applied unconditionally in training and sampling.

\subsection{Time Conditioning}
Since the expanding flow map $\Phi_{s,t}$ transports a sequence between two times rather than denoising at a single one, the network must condition on the source and target times $(s,t)$. We do this through two separate sinusoidal time embedders, whose outputs are summed and injected into every transformer block via adaptive layer normalization (adaLN). The final projection of the target embedder is zero-initialized so that at the start of training the target time has no effect and the model behaves as a standard single-time denoiser conditioned on $s$, which we found stabilizes the early distillation phase before the two-time flow-map behavior is learned.

Since positions are inserted at different global times $t^{\text{ins}}_i$ and have different local noise levels $\boldsymbol{t}_{\text{local}}=(t_1, \dots, t_{d(t)}, 0, \dots, 0)$ at any global time $t$, we apply \textit{source-time} adaLN conditioning per position. Each active position's local time $t_i$ is passed through the same sinusoidal source-time embedder used for the global source time $s$, producing a per-position conditioning vector $c_i$ that modulates every transformer block and the output layer via adaLN, while the target-time embedding stays global and is broadcast across positions. Conditioning each position on its own local time makes every position behave similarly to the teacher's single-time denoiser evaluated at noise level $t_i$, which keeps the flow-map student consistent with the fixed-length denoiser it distills from. Since the state is mixed with the (possibly non-linear) coefficient $\beta$, the true local time fed to this embedder is recovered by inverting the mixing schedule, $t_i=\beta^{-1}(1-\sigma_i)$, so the conditioning is exact even under non-linear mixing rather than reading $t_i$ off the noise level directly.

\subsection{Adaptive Loss Weighting}
To stabilize training, we scale the distillation losses by a weight that lowers the contribution of large mismatches between the student and teacher models \citep{potaptchik2026discrete}. Given the student mean denoiser $\psi_{s,t}(\bx)\in \Delta^{V-1}$ and the target $\bar \psi_{s,t}(\bx)\in \Delta^{V-1}$, the weight is defined as:
\begin{align}
    w_{s,t}(\bx):=\texttt{sg}\left[\left(\|\Delta_{s,t}(\bx)\|^2+c\right)^{-r}\right], \quad \Delta_{s,t}(\bx):=\psi_{s,t}(\bx)-\bar \psi_{s,t}(\bx)
\end{align}
where $\Delta_{s,t}$ is the mismatch between the categorical distributions and $c$ and $r$ are tunable hyperparameters. This is used to scale the cross-entropy loss defined in (\ref{eq:cross-entropy-loss}), where we set $c = 1e-6$ and $r = 0.5$.

\subsection{One-Step Sampling}
\label{app:one-step}
Under the few-step sampling algorithm for discrete EFM, one-step sampling results in degenerate behavior since in a single step, it determines the number of tokens to insert by sampling from the Binomial distribution in (\ref{eq:insertion-dist}), and then takes a jump with the flow map $\Phi_{0,1}(\cdot)$ over time $t\in [0,1]$ for all inserted tokens given the fully noisy sequence. However, since the probability of the insertion time of a token being 1 is near zero, the model is never trained on the case where all tokens are newly inserted noise samples at time $t=0$ and therefore defaults to the mode of the distribution when sampling from $\Phi_{0,1}(\bx_0)$.

To overcome this, we define the denoising interval as $[t_{\text{denoise-min}}, 1]$ and the one-step map as $\Phi_{0,1}(\bx_0)=\psi_{t_{\text{denoise-min}},1}(\mathcal{E}_{0,1}(\bx_0))$. This allows the model to denoise from a noisy sequence at a time step $t_{\text{denoise-min}}$, which resembles the time at which the model has seen sequences of this form during training. This ensures that the denoiser is never conditioned on the near-pure-noise full-length input at time $t=0$ it rarely encounters in training, while the insertion schedule still uses the true time.

\section{Discrete Graph Implementation Details}
\label{app:graph-algo}

Here we describe the graph implementation of EFlow and EFM. The construction mirrors the sequence variant. Each \emph{node} has an independent insertion time, the clean target is interpolated with a Gaussian-latent prior, a learned head fires per-gap Poisson insertions, and the denoiser is rolled out with an Euler scheme on the resulting flow ODE. The one difference is that the state is a pair $(\boldsymbol{x}, \boldsymbol{E})$ with a square edge matrix $\boldsymbol{E}$. Every position update is applied consistently on the node axis and both edge axes, and the edge matrix stays symmetric with zero diagonal at all times.

\subsection{Notation}
A clean graph with $d(1)$ nodes lives on the discrete space of \textit{nodes} $\mathcal{X}_x$ and edges $\mathcal{X}_e$,
\begin{align}
\mathcal{X}_x = \{0,\dots,V_x-1\}^{d(1)}, \quad \mathcal{X}_e \in \{0,\dots,V_e-1\}^{d(1)\times d(1)},
\end{align}
where $V_x$ counts node categories and $V_e$ counts edge categories (with $0$ reserved for the absence of an edge), the graph analogs of sequence vocabularies. Note that $\mathcal{X}_x$ carries no ``no-node'' class. The node count is a property of the active set, controlled entirely by the insertion process below and never by an absorbing node category. We represent a clean graph as one-hot vectors $\bx^i_1 \in \Delta^{V_x-1}$ at each node $i$ and $\be_1 \in \Delta^{V_e-1}$ at each ordered pair $(i,j)$. The full graph is the tensor of nodes $\boldsymbol{x}_1\in \mathbb{R}^{d(1)\times V_x}$ and edges $\boldsymbol{E}_1\in \mathbb{R}^{d(1)\times d(1)\times V_e}$, with the edge matrix symmetric and zero on the diagonal, $\boldsymbol{E}_1=\boldsymbol{E}^\top_1$ and $\mathrm{diag}(\boldsymbol{E}_1)=0$.

\subsection{Per-node Insertion Times and Edge Activation}
Each node $i$ is assigned an insertion time $t^{\text{ins}}_i \in [0,1]$ drawn from the schedule's inverse CDF
\begin{equation}
t^{\text{ins}}_i = t(\tau_i), \quad \tau_i \sim \mathrm{Uniform}(0,1)
\label{eq:tau}
\end{equation}
For padded slots, we set $t^{\text{ins}}_i = 2$ so they never activate. At global time $t$, a node is \emph{active} if its insertion time has elapsed,
\begin{equation}
a_i(t) = \mathbf{1}\{t^{\text{ins}}_i \le t\}\in\begin{cases}
    1&\text{(active)}\\
    0&\text{(inactive)}
\end{cases}
\label{eq:active-node}
\end{equation}
and its per-node local time and noise level are
\begin{equation}
t_i = \tfrac{t-t^{\text{ins}}_i}{1-t^{\text{ins}}_i}\cdot a_i(t),
\quad
\sigma_i =1 - t_i.
\label{eq:local-time}
\end{equation}
An edge inherits its noise level from the nodes it connects. An ordered pair $(i,j)$ with $i\ne j$ is active when both endpoints are active and inherits its local time as the \emph{minimum} of the endpoint local times,
\begin{equation}
a_{ij}(t) = a_i(t)a_j(t)(1-\delta_{ij}), \quad t_{ij} = \min\left(t_i,t_j\right)
\label{eq:edge-active}
\end{equation}
In the sequence EFM, the only activation variable is per-position
$a_i$. In the graph EFM, edges are \emph{coupled} to nodes through~\eqref{eq:edge-active}, so insertions of nodes implicitly grow the active edge set, and an edge can never be denoised ahead of either endpoint. 

\subsection{Node and Edge Interpolants}
Both nodes and edges are corrupted by the same flow-matching interpolant as in the sequence case, defined as a Gaussian-latent path between a one-hot target and a Gaussian prior $\mathcal{N}(\boldsymbol{0},\sigma^2 \boldsymbol{I})$. For nodes we draw $\bx_0 \sim \mathcal{N}(\boldsymbol{0}, \sigma^2\boldsymbol{I}_{V_x})$ and use the linear interpolant
\begin{align}
    \bx_t^i=t_i\bx_1^i+(1-t_i)\bx_0^i, \quad i \in \{1, \dots, d(1)\}
\end{align}
where inactive nodes have $t_i=0$ and stay fixed at pure noise. For edges we draw $\boldsymbol{E}_0\sim \mathcal{N}(\boldsymbol{0}, \sigma^2\boldsymbol{I}_{V_e})$, symmetrize it as $\tilde{\boldsymbol{E}}_0=(\boldsymbol{E}_0+\boldsymbol{E}_0^\top) /2$, and use the interpolant
\begin{align}
    \boldsymbol{E}_t^{ij}=t_{ij}\boldsymbol{E}_1^{ij}+(1-t_{ij})\boldsymbol{E}_0^{ij}, \quad\tilde{\boldsymbol{E}}_t:= \frac{1}{2}(\boldsymbol{E}_t+\boldsymbol{E}^\top_t), \quad \text{diag}(\tilde{\boldsymbol{E}}_t)=0
\end{align}
where the edge local time $t_{ij}=\min(t_i,t_j)$ is inherited from its endpoints as in~\eqref{eq:edge-active}, which symmetrizes the edge matrix and sets the diagonal (indicating self-edges) to zero at every step.

\subsection{Compacting the Partially Active Graph}
To form the partial graph with active positions compacted to be contiguous, we \texttt{gather} on the position vector and on both axes of the edge matrix. A position $j$ in the full graph falls in the gap $\gamma(j)$ given by
\begin{align}
    \gamma(j)=\sum_{k < j}a_{k}(t)
\end{align}
and the number of uninserted nodes at each gap $i$ at time $t$ is
\begin{align}
    g_i(t)=\sum_{j:a_j(t)=0}\boldsymbol{1}\{\gamma(j)=i\}, \quad i\in \{0,\dots, d(t)\}\label{eq:gap-counts}
\end{align}
which defines the prediction target for the insertion head during training.

\subsection{Node and Edge Insertions}
Given the compacted graph $(\bx_s,\boldsymbol{E}_s)$ with $d(s)$ active nodes and the per-gap counts $g_i(t)$ of (\ref{eq:gap-counts}), a shared insertion head reads the node hidden states and predicts the two-time insertion expectation $\hat{\mathcal{I}}_{s,t}(\bx_s)[i]$ for each gap $i\in\{0,\dots,d(s)\}$, trained exactly as in the sequence case (\ref{eq:insert-loss}) against $g_i(t)$. The \textit{graph expand operator} $\mathcal{E}_{s,t}$ jumps from $d(s)$ to $d(t)$ nodes by drawing a per-gap count from the binomial law (\ref{eq:insertion-dist}),
\begin{align}
\ell_i \sim \text{Binomial}\Big(d_{\max}-d(s),\ \tfrac{\hat{\mathcal{I}}_{s,t}(\bx_s)[i]}{d_{\max}-d(s)}\Big),
\end{align}
capped left-to-right at the node budget so $\sum_i\ell_i\le d_{\max}-d(s)$, and converging to $\ell_i\sim\text{Poisson}(\hat{\mathcal{I}}_{s,t}(\bx_s)[i])$ in the many-step limit.

Unlike sequences, an inserted node carries \emph{both} a node coordinate and a full row and column of edges. Writing $A$ for the $d(s)$ existing nodes and $N$ for the $\sum_i\ell_i$ freshly inserted nodes, the expand operator scatters the old state into the enlarged $d(t)$-node layout and fills every new coordinate with a Gaussian latent at local time $t_i=0$:
\begin{align}
\bx^\epsilon_s[i]=\begin{cases}\bx_s[i] & i\in A\\[2pt] \boldsymbol{z}\sim\mathcal{N}(\boldsymbol{0},\sigma^2\boldsymbol{I}_{V_x}) & i\in N\end{cases},
\quad
\boldsymbol{E}^\epsilon_s[i,j]=\begin{cases}\boldsymbol{E}_s[i,j] & i,j\in A\\[2pt] \tilde{\boldsymbol{e}}\sim\mathcal{N}(\boldsymbol{0},\sigma^2\boldsymbol{I}_{V_e}) & i\in N \text{ or } j\in N\end{cases}
\end{align}
so the existing $A\times A$ edge block is copied while every pair touching a new node (the $A\times N$, $N\times A$, and $N\times N$ blocks) is seeded with noise. The result is symmetrized and its diagonal zeroed, $\boldsymbol{E}^\epsilon_s\gets\tfrac12(\boldsymbol{E}^\epsilon_s+\boldsymbol{E}^{\epsilon\top}_s)$ with $\mathrm{diag}(\boldsymbol{E}^\epsilon_s)=0$, so the augmented state stays a valid symmetric graph. Because a graph has \emph{no inherent node ordering}, the per-gap counts are needed only to define the training target against the node order of (\ref{eq:gap-counts}), and at sampling any placement of the $\sum_i\ell_i$ new nodes yields the same graph up to relabeling. This is the graph instance of positional insertion collapsing to concatenation under permutation invariance.

\subsection{Denoising Nodes and Edges}
The transport map $X_{s,t}$ denoises the augmented graph $(\bx^\epsilon_s,\boldsymbol{E}^\epsilon_s)$ toward the clean one-hot graph. A single graph transformer $\hat\psi_{s,t}$ takes the augmented state, the target time $t$, and the per-node source local times, emitting a mean denoiser on the simplex at every node and edge:
\begin{align}
\hat\psi^X_{s,t}(\bx^\epsilon_s)[i]\in\Delta^{V_x-1}, \quad \hat\psi^E_{s,t}(\boldsymbol{E}^\epsilon_s)[i,j]\in\Delta^{V_e-1}.
\end{align}
Node $i$ is conditioned on its own local time $s_i$ and edge $(i,j)$ on the inherited $s_{ij}=\min(s_i,s_j)$ of (\ref{eq:edge-active}), so a more recently inserted coordinate is denoised from a higher noise level. The discrete expanding flow map is then applied identically on the node axis and on both edge axes,
\begin{align}
\Phi^X_{s,t}(\bx^\epsilon_s)=\tfrac{1-t}{1-s}\bx^\epsilon_s+\tfrac{t-s}{1-s}\hat\psi^X_{s,t}(\bx^\epsilon_s), \quad \Phi^E_{s,t}(\boldsymbol{E}^\epsilon_s)=\tfrac{1-t}{1-s}\boldsymbol{E}^\epsilon_s+\tfrac{t-s}{1-s}\hat\psi^E_{s,t}(\boldsymbol{E}^\epsilon_s),
\end{align}
with the same global $(s,t)$ coefficients as the discrete sequence implementation. After each jump, the edge tensor is re-symmetrized and its diagonal zeroed, $\boldsymbol{E}_t\gets\tfrac12(\boldsymbol{E}_t+\boldsymbol{E}_t^\top)$ with $\mathrm{diag}(\boldsymbol{E}_t)=0$, preserving symmetry and the absence of self-edges. Inactive nodes and edges are masked and stay at their prior noise until inserted.

Training uses the cross-entropy discrete EFM objective (\ref{eq:cross-entropy-loss}) summed over the active node set $a_i=1$ and active edge set $a_{ij}=1$ of (\ref{eq:active-node})-(\ref{eq:edge-active}), with the two channels balanced by an edge weight $\lambda_E$. 
\begin{align}
\mathcal{L}_{\text{diag}}(\hat{\psi})&=\mathbb{E}_t\mathbb{E}_{\bx_0,\bx_1}\bigg[-\sum_{i:a_i=1}\big(D^X_{t}\cdot\log\hat\psi^X_{t,t}\big)(\bx^\epsilon_t)[i]-\lambda_E\!\!\sum_{(i,j):a_{ij}=1}\!\!\big(D^E_{t}\cdot\log\hat\psi^E_{t,t}\big)(\boldsymbol{E}^\epsilon_t)[i,j]\bigg]\label{eq:graph-diag-loss}\\
\mathcal{L}_{\text{cons}}(\hat{\psi})&=\mathbb{E}_{s,u,t}\mathbb{E}_{\bx_0,\bx_1}\bigg[-\sum_{i:a_i=1}\big(\bar\psi^X_{s,t}\cdot\log\hat\psi^X_{s,t}\big)(\bx^\epsilon_s)[i]-\lambda_E\!\!\sum_{(i,j):a_{ij}=1}\!\!\big(\bar\psi^E_{s,t}\cdot\log\hat\psi^E_{s,t}\big)(\boldsymbol{E}^\epsilon_s)[i,j]\bigg]\label{eq:graph-cons-loss}
\end{align}
where $D^{\{X,E\}}_{t}$ are the instantaneous denoising target for $s=t$ and $\bar\psi^{\{X,E\}}_{s,t}$ is the semigroup consistency target defined in (\ref{eq:semigroup-target}) for the off-diagonal ($s<t$) pairs. Because the denoiser is permutation-equivariant and every update acts identically on the node and both edge axes, the full map is invariant to node relabeling, matching the unordered structure of the target graph.

\section{Experiment Details}

\subsection{Coarse-to-Fine Molecular Conformer Generation}
\label{app:conformer}
\paragraph{Datasets}
We conduct our experiments on the GEOM dataset \citep{axelrod2022geom}, which comprises GEOM-QM9 and GEOM-Drugs. Following the split of GeoDiff \citep{xu2022geodiff}, each dataset provides $40{,}000$ molecules for training and $5{,}000$ for validation, each with $5$ conformers. For testing we use $200$ molecules per dataset, yielding $22{,}409$ conformers for QM9 and $14{,}324$ for Drugs. Atoms are encoded by atomic number and bonds by their $4$ RDKit types (single, double, triple, aromatic). For each test molecule, we generate twice its number of reference conformers, matching the GeoDiff/MSGEN evaluation protocol \citep{huhierarchical}.

\paragraph{Framework Setting}
Our main experiments adopt a coarse-to-fine two-stage decomposition of each molecule into a heavy-atom backbone and its hydrogen atoms. This reflects the fact that heavy atoms form the structural scaffold while hydrogens govern local stereochemistry and fine geometry, so the backbone is generated first to serve as a global spatial anchor that guides hydrogen placement. EFlow naturally implements this hierarchical structure within a single expanding flow: the heavy backbone is denoised from Gaussian noise and frozen by $t_{\text{ins-end}}$, and then the hydrogens are inserted on a schedule and resolved around the fixed backbone. An optional low-noise refiner forms a second stage that polishes the generated all-atom structure.

\paragraph{Model Architecture}
For fair comparison, the velocity field is the GeoDiff dual-encoder \citep{xu2022geodiff} used by MSGEN \citep{huhierarchical}, which is a global SchNet over the extended-order and radius graph ($6$ interactions), a local GIN over the bond graph ($4$ convolutions), higher-order extended edges (\texttt{edge\_order}$=3$), and a radius cutoff of $10$ \AA, all with hidden dimension $128$. Each encoder emits a per-edge invariant scalar that \texttt{eq\_transform} maps to an equivariant per-node velocity, and the two are summed ($v = v_{\text{global}} + v_{\text{local}}$). The network is conditioned on the global time pair $(s,t)$ and on the per-atom local (insertion) time through Fourier embeddings ($\dim=64$) added to the edge features. The refiner uses the identical architecture.

\paragraph{Training Details}
The interpolant is the per-position Gaussian-latent linear interpolant $\bx_s^i = (1-t_i) \bx^i_0 + t_i \bx_1^i$, where each active atom has local time $t_i\in[0,1]$. Heavy atoms draw $\bx^i_0\sim\mathcal{N}(0,\sigma^2_z\boldsymbol{I})$ with $\sigma_z=1.0$. Each hydrogen is anchored on its bonded heavy parent as $\bx^i_0= \bx^{\text{parent}(i)}_0 + \sigma_h \boldsymbol{\epsilon}$ with $\sigma_h=0.1$, and positions are made center-of-mass free at every step. Heavy atoms are present from $t=0$ with denoise window $[0, t_{\text{ins-end}}]$, and each hydrogen is inserted at a per-atom time $t^{\text{ins}}_i$ drawn from a cosine schedule on $[0, t_{\text{ins-end}}]$ with window $[t^{\text{ins}}_i, 1]$. With $t_{\text{ins-end}}=0.6$ the heavy backbone forms and freezes by $t=0.6$, leaving a hydrogen-resolving tail on $(0.6, 1]$. The model predicts the mean velocity $v_{s,t}$ and is trained with the per-graph mean-squared error over inserted atoms, with $5{,}000$ diagonal-only flow-matching warmup steps and the off-diagonal range annealed to full by $20{,}000$ steps.

EFlow is trained with the AdamW optimizer, learning rate $5\times10^{-4}$, a cosine schedule to $10^{-6}$, gradient clipping at $1.0$, an EMA of the weights with decay $0.999$ (used for sampling), batch size $256$ for QM9 and $128$ for Drugs, for $300$ epochs on a single NVIDIA A100 GPU. The refiner is a separate single-scale low-noise all-atom denoiser trained with source $\bx_0 = \bx_1 + \sigma_r \boldsymbol{\epsilon}$ (all atoms present), where $\sigma_r$ is approximated by the error of the initial insertion-denoising process ($\sigma_r=0.4$ for QM9, $1.0$ for Drugs). 

\paragraph{Sampling}
EFlow (no refinement) generates conformers by Euler-integrating the per-position flow ODE on a uniform grid of $20$ steps. Insertion follows the deterministic per-atom schedule (heavy atoms present from $t=0$, each hydrogen inserted at its $t^{\text{ins}}_i$) and costs no additional network evaluation; heavy atoms are frozen once $t\geq t_{\text{ins-end}}$. Because the molecular graph fixes the atom count, the expand operator is the deterministic case of the general framework, and the per-atom times $t^{\text{ins}}_i$ are drawn once from the insertion schedule before integration, so no learned insertion head is required and no count has to be inferred at sampling time. EFlow+R (with refinement) appends a $10$-step refiner pass that takes the output of the EFlow or EFM model as its initialization and integrates all atoms over $[0,1]$.

EFM samples with the same insertion schedule but replaces the Euler steps with two-time flow-map jumps $\Phi_{s,t}$ on a grid of $1$, $2$, or $4$ steps, so the number of network evaluations drops by an order of magnitude while the interpolant is unchanged. EFM+R (with refinement) applies the $4$-step flow map followed by $10$-step refinement.

\paragraph{Baselines}
We build EFlow on top of GeoDiff \citep{xu2022geodiff}, used as the backbone network, and denote the hierarchical two-stage variant GeoDiff+MSGEN \citep{huhierarchical}. In the few-step generation setting, we evaluate GeoDiff and SubgDiff \citep{zhang2024subgdiff} at $500$ denoising steps and GeoDiff+MSGEN \citep{huhierarchical} at two $500$-step stages. For broader comparison, we include deep generative models GraphDG \citep{simm2019generative}, CGCF \citep{xu2021learning}, ConfVAE \citep{xu2021end}, GeoMol \citep{ganea2021geomol}, and ConfGF \citep{shi2021learning}. Reported baseline values are taken from \citet{xu2022geodiff} and \citet{zhang2024subgdiff}.

\paragraph{Evaluation}
We adopt the widely used Coverage (COV) and Average Minimum RMSD (AMR) metrics proposed by \citet{ganea2021geomol}, in both recall (R) and precision (P) forms. Both are computed from the root-mean-square deviation (RMSD) between conformers aligned by the Kabsch algorithm \citep{kabsch1976solution}, following the GeoDiff/MSGEN protocol \citep{xu2022geodiff, huhierarchical}. Given a threshold $\delta$ ($0.5$ \AA{} for QM9 and $1.25$ \AA{} for Drugs), COV-R is the fraction of reference conformers matched within $\delta$ by some generated conformer, and AMR-R is the average minimum RMSD to the nearest generated conformer; the precision variants swap the roles of the generated and reference sets. Each metric is reported as the mean and median over test molecules.

\paragraph{Property Prediction Metrics}
This benchmark aims to evaluate whether a generated conformer set reproduces a molecule's \emph{ensemble-averaged} quantum-chemical behavior \citep{axelrod2022geom} rather than the geometry of any single conformer. Following \citet{shi2021learning}, we generate $50$ conformers each for $30$ test molecules from GEOM-QM9 and obtain the electronic energy and its HOMO and LUMO levels using the quantum chemistry calculation package, Psi4 \citep{smith2020psi4}. Using these values, we compute five properties for each molecule: the mean and minimum conformer energies $\bar{E}$ and $E_{\min}$, mean $\overline{\Delta\varepsilon}$, minimum $\Delta\varepsilon_{\min}$, and maximum $\Delta\varepsilon_{\max}$ HOMO-LUMO gap $\Delta\varepsilon = |\varepsilon_{\text{LUMO}} - \varepsilon_{\text{HOMO}}|$. Computing the same descriptors from each molecule's reference ensemble, we score a model by the absolute gap between its predicted and reference values, and take the median over the $30$ molecules. Results are reported in Table \ref{table:conformer-property-median} in eV, where lower is better.

\begin{table}[h!]
\centering
\caption{Ensemble property prediction on GEOM-QM9, reported as the \emph{median} of absolute prediction errors (eV; lower is better). Baseline medians (RDKit, ConfGF) are from
\citet{shi2021learning}.}
\label{table:conformer-property-median}
\begin{tabular}{l c ccccc}
\toprule
Model & Steps $\downarrow$ & $E$ $\downarrow$ & $E_{\min}$ $\downarrow$ & $\Delta\varepsilon_{\text{avg}}$ $\downarrow$ & $\Delta\varepsilon_{\min}$ $\downarrow$ & $\Delta\varepsilon_{\max}$ $\downarrow$ \\
\midrule
RDKit  & 500 & 0.8914 & 0.6629 & 0.2947 & \textbf{0.5196} & 0.1617 \\
ConfGF & 500 & 0.5328 & 0.1145 & 0.3207 & 0.7365 & 0.1337 \\
\midrule
\rowcolor{mybg} EFlow+R & 30 & \textbf{0.4506} & \textbf{0.0731} & \textbf{0.2204} & 0.8717 & \textbf{0.1156} \\
\rowcolor{mybg} EFM+R & 14 & 2.0521 & 0.1232 & 0.3208 & 2.6383 & 0.1672 \\
\bottomrule
\end{tabular}
\end{table}

\subsection{Discrete Molecular Graph Generation}
\label{app:molecule}
\paragraph{Dataset}
We train EFlow and EFM on the QM9 dataset \citep{wu2018moleculenet} containing small molecules with up to $9$ heavy atoms. Instantiating the state space of App \ref{app:graph-algo}, the node categories are the $V_x=4$ heavy-atom types (C, N, O, F) and the edge categories are the $V_e=5$ bond types (no-bond, single, double, triple, aromatic). The node buffer is $d_{\max}=9$. Following \citet{roos2026categorical}, we split the data into 100K molecules for training, 20K for validation, and 13K for testing.

\paragraph{Model Architecture}
For fair comparison, we use the same graph Transformer architecture as DeFoG \citep{qin2024defog}, with $9$ layers with node/edge/global hidden dimensions $(d_X, d_E, d_y) = (256, 64, 64)$, $8$ attention heads, and FFN hidden dimensions $(256, 128, 128)$. The input MLPs project to dimensions $(256, 128, 128)$. The RRWP structural features are computed with $k=12$ random-walk steps.

\paragraph{Training Details}
Both the EFlow teacher and the EFM flow-map student are trained with batch size $256$, AdamW (weight decay $10^{-12}$), learning rate $3\times10^{-4}$, constant schedule with $1,000$ warmup steps, and a maximum of $2{\times}10^6$ training steps on a single NVIDIA A100 GPU.

The Gaussian-latent interpolant of App \ref{app:graph-algo} is instantiated with prior scale $\sigma=1.0$ on both the node and edge axes. Training uses an antithetic time sampler with sampling-time floor $\varepsilon_t=10^{-3}$ and edge-loss weight $\lambda_E=5.0$. The student is initialized from the teacher checkpoint and, on each step, we sample $s=t$ with probability $0.75$, training on the diagonal loss (\ref{eq:graph-diag-loss}) against the frozen EFlow teacher's endpoint prediction, and $s<t$ with probability $0.25$, training on the consistency loss (\ref{eq:graph-cons-loss}) against the target $\hat\psi_{s,t}$. 

Since the size of QM9 is small, we train the learned per-gap insertion head with the diagonal loss and drop the off-diagonal term of (\ref{eq:insert-loss}), so the two-time count is reconstructed at sampling time rather than learned. On a jump $s\to t$ we scale the diagonal head by the schedule fraction $\rho_{s,t}=\tfrac{\alpha_{t}-\alpha_s}{1-\alpha_s}$. We use a polynomial insertion schedule $\rho(t) = (r t^{r-1})/(1-t^{r})$ with exponent $r=0.5$ for $t\in [10^{-4}, 1.0]$ without an explicit insertion cutoff ($t_{\text{ins-end}}=1$) since the polynomial schedule with $r=0.5$ is concave with median insertion time $t^{\text{ins}}_i=0.25$.

\paragraph{Sampling}
EFlow generates by Euler-integrating the per-position flow ODE on a uniform grid of $N_{\text{steps}}$ steps, applying the learned insertion head before each denoiser step to grow the graph. Per-gap insertion counts are drawn from the Poisson law $\ell_i\sim\text{Poisson}(\hat{\mathcal{I}}_{s,t}(\bx_s,\boldsymbol{E}_s)[i])$. EFM uses the same grid with the flow map in place of the denoiser, and at $N_{\text{steps}}=1$ reduces to a single call $\psi_{0,1}$ preceded by one insertion-head call at $t=0$. Reported numbers are $10{,}000$ molecules per step count.

\paragraph{Baselines}
We compare against DeFoG \citep{qin2024defog} by running their publicly released checkpoints using the same model architecture and sampling steps on $10,000$ molecules per step count and categorical flow maps (CFM) \citep{roos2026categorical} via their reported results on the same model architecture and sampling steps. 

\paragraph{Evaluation}
To assess sample quality, we compute the proportion of $10{,}000$ generated graphs per method and step count that correspond to RDKit-sanitizable molecules with valid SMILES (Validity), the proportion of these that are distinct under largest-fragment SMILES canonicalization (Uniqueness), and the Fréchet ChemNet Distance (FCD) between generated and reference molecules in ChemNet's final-layer embedding space against the test dataset. 

\subsection{Language Modeling}
\label{app:language}
\paragraph{Dataset}
We evaluate on the One Billion Word Benchmark (LM1B) \citep{chelba2013one}, a standard corpus for large-scale language modeling. We use the \texttt{bert-base-uncased} tokenizer with vocab size $V=30,522$. 

\paragraph{Model Architecture}
The architecture follows the DDiT backbone used in \citet{sahoo2025diffusion}, consisting of $12$ blocks, hidden size $768$, $12$ attention heads, conditioning dimension $128$, and dropout $0.1$. The insertion head described in App \ref{app:discrete-efm-implementation} reads the same $768$-dimensional pre-output hidden states and is conditioned by the same $128$-dimensional time vector, adding $0.79$M parameters on top of the backbone.

\paragraph{Training Details}
The EFlow denoising backbone is warm-started from the publicly released LM1B FLM checkpoint of \citet{lee2026flow}, but our framework also allows training from scratch. We emphasize that while the warm start accelerates convergence, the variable-length, per-position denoising and insertion objectives diverge from those used to train FLM. 

The insertion head is randomly initialized and, for the first $2{,}000$ steps, its gradients are detached from the backbone so that the untrained head does not corrupt the warm-started weights. Training then proceeds for $2{\times}10^5$ steps on $4$ NVIDIA B200 GPUs with global batch size $512$ (per-GPU batch size of $128$, no gradient accumulation), AdamW with $\beta_1=0.9$, $\beta_2=0.999$, $\varepsilon=10^{-8}$, no weight decay, learning rate $3{\times}10^{-4}$ with a constant schedule after $2{,}500$ warmup steps, gradient clipping $1.0$, and EMA decay $0.9999$. 

The per-position insertion times follow a cosine schedule, whose cumulative insertion fraction $\alpha(t)=1-\cos\!\big(\tfrac{\pi}{2}t/t_{\text{ins-end}}\big)$ is confined to the window $[0,t_{\text{ins-end}}]$, so all positions are inserted by the cutoff (median insertion time $\tfrac{2}{3}t_{\text{ins-end}}$) and $[t_{\text{ins-end}},1]$ is left for pure denoising. The insertion head is trained with insertion cutoff time $t_{\text{ins-end}}=0.5$, input noise scale $1.25$, and per-position adaLN time conditioning.

For EFM, we distill from the same frozen FLM teacher \citep{lee2026flow}, but our framework also allows training from self-distillation or from an EFlow teacher. We use a diagonal fraction $0.75$, boundary probability $1/32$, and uniform-difference off-diagonal $(s,t)$ sampling. The insertion head learns the binomial interval insertion law, with a $2000$-step detached warmup. The two loss terms are merged by gradient surgery following \citet{potaptchik2026discrete}. 

\paragraph{Baselines}
We compare against both multi-step and few-step baselines. For multi-step baselines, we compare EFlow against the fixed-length flow language model (FLM) \citep{lee2026flow} at budgets from $64$ to $1024$ steps, evaluated by running their publicly released checkpoints. For the few-step baselines, we compare EFM at $1$, $2$, and $4$ steps against distilled diffusion and flow-map baselines, including Duo distilled with DCD \citep{sahoo2025diffusion}, MDLM with SDTT \citep{deschenaux2025beyond}, both Duo and MDLM distilled with Di4C, categorical flow maps (CFM) \citep{roos2026categorical}, and flow map language models (FMLM) \citep{lee2026flow}, with the numbers reported in \citet{lee2026flow}.

\paragraph{Evaluation}
For each step count, we generate $1024$ sequences under the EMA weights using Euler integration of the insertion-denoising process, take the argmax of the final vectors, decode with the \texttt{bert-base-uncased} tokenizer, and re-encode with the GPT-2 tokenizer. Then, we compute generative perplexity under GPT-2 Large \citep{radford2019language}. We additionally compute the sample entropy as the empirical token-frequency entropy per sample.

\clearpage
\section{Algorithms}
\label{app:algorithms}

\begin{algorithm}[h!]
\caption{Training \textbf{Discrete Expanding Generative Flows} (EFlows)}\label{alg:training-eflow}
    \begin{algorithmic}[1]
        \State \textbf{Input:} Dataset $\mathcal{D}$, insertion schedule $\alpha$, time schedule $t(\cdot)$, initialized networks $\hat{D}, \hat{\mathcal{I}}$, insertion loss weight $\lambda_{\text{insert}}$, maximum length $L$
        \While{not converged}
        \State $\boldsymbol{x}_1\sim \mathcal{D}, \boldsymbol{x}_0\sim \mathcal{N}(\boldsymbol{0}, \boldsymbol{I}_{d(0)})$
        \State $\tau\sim\mathcal{U}(0,1)$, $t\gets t(\tau)$\Comment{sample global time via time schedule}
        \For{$i$ in $1, \dots, L$}
        \State $t^{\text{ins}}_i\sim\dot{\alpha}_sds$\Comment{sample per-token insertion time}
        \State $t_i\gets \max\!\left(0, \tfrac{t-t^{\text{ins}}_i}{1-t^{\text{ins}}_i}\right)$\Comment{local time induced by global $t$}
        \State $\boldsymbol{x}^i_{t_i}\gets (1-t_i)\boldsymbol{x}^i_0+t_i\boldsymbol{x}^i_1$\Comment{sample expanding interpolant}
        \EndFor
        \State $\boldsymbol{x}_t\gets(\boldsymbol{x}^i_{t_i}: t^{\text{ins}}_i\leq t)$\Comment{active subsequence, $|\boldsymbol{x}_t|=d(t)$}
        \State $\hat{\bx}_1\gets \hat D_t(\bx_t)$
        \State $\mathcal{L}_{\text{CE}}\gets-\sum_{i: t^{\text{ins}}_i\leq t}\bx_1^i\cdot\log  \hat{\bx}_1^i$
        \State $\mathcal{L}_{\text{insert}}\gets\frac{\dot{\alpha}_t}{1-\alpha_t}\sum_{i=0}^{d(t)}\phi(\text{ind}_i (t) -\text{ind}_{i- 1}(t)-1, \hat{\mathcal{I}}_t(\boldsymbol{x}_t))$
        \State $\mathcal{L}_{\text{total}}\gets\mathcal{L}_{\text{CE}}+\lambda_{\text{insert}}\mathcal{L}_{\text{insert}}$
        \State Update $\hat{\mathcal{I}}, \hat{D}$ using $\nabla\mathcal{L}_{\text{total}}$
        \EndWhile
        \State \textbf{return} $\hat{\mathcal{I}}$, $\hat{D}$
    \end{algorithmic}
\end{algorithm}

\begin{algorithm}[h!]
\caption{Training \textbf{Discrete Expanding Flow Maps} (EFMs)}\label{alg:training-efm}
    \begin{algorithmic}[1]
        \State \textbf{Input:} Dataset $\mathcal{D}$, insertion schedule $\alpha$, time schedule $t(\cdot)$, initialized networks $\hat{\psi}, \hat{\mathcal{I}}$, insertion loss weight $\lambda_{\text{insert}}$, maximum length $L$
        \While{not converged}
        \State $\boldsymbol{x}_1\sim \mathcal{D}, \boldsymbol{x}_0\sim \mathcal{N}(\boldsymbol{0}, \boldsymbol{I}_{d(0)})$
        \State $\tau_s,\tau_t \sim \mathcal{U}(0,1)$ with $\tau_s\leq\tau_t$, \; $\tau_u\gets\tfrac{1}{2}(\tau_s+\tau_t)$\Comment{ordered global times}
        \State $s,u,t\gets t(\tau_s), t(\tau_u), t(\tau_t)$\Comment{map through the time schedule}
        \For{$i$ in $1, \dots, L$}
        \State $t^{\text{ins}}_i\sim\dot{\alpha}_sds$\Comment{sample per-token insertion time}
        \State $s_i\gets \max\!\left(0, \tfrac{s-t^{\text{ins}}_i}{1-t^{\text{ins}}_i}\right)$\Comment{local time induced by global $s$}
        \State $\boldsymbol{x}^i_{s_i}\gets (1-s_i)\boldsymbol{x}^i_0+s_i\boldsymbol{x}^i_1$\Comment{sample expanding interpolant}
        \EndFor
        \State $\boldsymbol{x}_s\gets(\boldsymbol{x}^i_{s_i}: t^{\text{ins}}_i\leq s)$\Comment{active subsequence, $|\boldsymbol{x}_s|=d(s)$}
        \State $\boldsymbol{x}_s^\epsilon\gets(\boldsymbol{x}^i_{s_i}: t^{\text{ins}}_i\leq t)$\Comment{expand once to $d(t)$ with gaps over $(s,t]$ as noise ($s_i{=}0$)}
        \State $\omega_{s,u,t}\gets\tfrac{(u-s)(1-t)}{(t-s)(1-u)}$
        \State $\hat{\Phi}^{(t)}_{s,u}(\boldsymbol{x}_s^\epsilon)\gets\tfrac{1-u}{1-s}\boldsymbol{x}_s^\epsilon+\tfrac{u-s}{1-s}\hat{\psi}^{(t)}_{s,u}(\boldsymbol{x}_s^\epsilon)$\Comment{flow map over $[s,u]$ at dimension $d(t)$}
        \State $\bar{\psi}_{s,t}(\boldsymbol{x}_s^\epsilon)\gets\text{sg}\left[\omega_{s,u,t}\hat{\psi}^{(t)}_{s,u}(\boldsymbol{x}_s^\epsilon)+(1-\omega_{s,u,t})\hat{\psi}_{u,t}\left(\hat{\Phi}^{(t)}_{s,u}(\boldsymbol{x}_s^\epsilon)\right)\right]$
        \For{$i$ in $0, \dots, d(s)$}
        \State $\mathcal{I}_{s}(\boldsymbol{x}_s)[i]\gets \mathbb{E}[\text{ind}_i (s) -\text{ind}_{i- 1}(s)-1\mid\boldsymbol{x}_s]$\Comment{per-gap missing count}
        \State $\mathcal{I}_{s,t}(\boldsymbol{x}_s)[i]\gets \tfrac{\alpha_t-\alpha_s}{1-\alpha_s} \mathcal{I}_s(\boldsymbol{x}_s)[i]$\Comment{expected insertions over $(s,t]$}
        \EndFor
        \State $\mathcal{L}_{\text{insert}}\gets\sum_{i=0}^{d(s)}\phi\!\left(\mathcal{I}_{s,t}(\boldsymbol{x}_s)[i],  \hat{\mathcal{I}}_{s,t}(\boldsymbol{x}_s)[i]\right)$
        \State $\mathcal{L}_{\text{DEFM}}(\hat{\mathcal{E}}, \hat{\psi})\gets\sum_{i=1}^{d(t)}\bigg[-(\bar{\psi}_{s,t}\cdot \log \hat{\psi}_{s,t})(\boldsymbol{x}_s^\epsilon)_i-(D_t\cdot \log \hat{\psi}_{t,t})(\boldsymbol{x}_t^i)\bigg]$
        \State $\mathcal{L}_{\text{total}}\gets \mathcal{L}_{\text{DEFM}}+\lambda_{\text{insert}}\mathcal{L}_{\text{insert}}$
        \State Update $\hat\psi, \hat{\mathcal{I}}$ using $\nabla \mathcal{L}_{\text{total}}$
        \EndWhile
        \State \textbf{return} $\hat \psi$, $\hat{\mathcal{I}}$
    \end{algorithmic}
\end{algorithm}

\begin{algorithm}[h!]
\caption{\texttt{ExpandOperator}: Insert tokens between tokens at parameterized rate}\label{alg:expand-operator}
    \begin{algorithmic}[1]
        \State \textbf{Input:} Intermediate subsequence $\bx_s$, global initial time $s$, global jump time $t$, maximum length $L$
        \State $n_s\gets \texttt{length}(\bx_s)$\Comment{get number of inserted tokens}
        \State $\hat{\mathcal{I}}_{s}(\bx_{s})\gets\text{InsertHead}(\bx_{s})$
        \State $\ell_i \sim \text{Binomial}\left(L - n_s,\frac{\hat{\mathcal I}_s(\bx_s)[i]}{L - n_s}\right)$\Comment{sample number of tokens to insert at each gap}
        \State $\boldsymbol{x}_0^{\ell_i}\sim \mathcal{N}(\boldsymbol{0}, \sigma^2\boldsymbol{I}_{\ell_i})$\Comment{sample noisy tokens}
        \State $\bx_{s}^\epsilon\gets \left[\bigoplus_{i=1}^{d(s)}\big(\bx_0^{\ell_i}\oplus\bx_{s_i}^i\big)\right]\oplus\bx_0^{\ell_{d(s)+1}}$\Comment{insert noisy tokens into sequence}
        \State \textbf{return} $\bx_{s}^\epsilon$
    \end{algorithmic}
\end{algorithm}

\begin{algorithm}[h!]
\caption{Sampling \textbf{Discrete Expanding Flow Maps} (EFMs)}\label{alg:sampling-efm}
    \begin{algorithmic}[1]
        \State \textbf{Input:} Trained expanding operator $\hat{\mathcal{E}}$, mean denoiser $\hat\psi$, time grid $\{t_k\}_{k=1}^K$, maximum length $L$
        \State $\bx_0\gets \text{(empty)}$
        \State $\bx_0\gets\hat{\mathcal{E}}_{0,t_{\text{denoise-min}}}(\bx_0)$
        \State $\bx_{t_1}\gets \hat \psi_{t_{\text{denoise-min}},t_1}(\bx_0)$\Comment{insert from minimum time}
        \For{$k$ in $1, \dots, K-1$}
        \State $\bx_{t_{k}}^\epsilon\gets \texttt{ExpandOperator}(\bx_{t_k},t_k,t_{k+1}, L)$\Comment{insert noisy tokens into sequence}
        \State $\bx_{t_{k+1}}\gets \hat\psi_{t_k,t_{k+1}}(\bx_{t_{k}}^\epsilon)$
        \EndFor
        \State $\bx_1\gets \arg\max(\bx_{t_K})$\Comment{sample one-hot argmax tokens}
        \State \textbf{return} $\bx_1$
    \end{algorithmic}
\end{algorithm}

\clearpage

\section{Example Generations}
\label{sec:example_generations}
\begin{figure}[h!]
\centering
\begin{figurepanel}

\noindent
\begin{minipage}[t]{0.485\linewidth}
\nfeheader{steps = 64}

\vspace{4pt}

\begin{samplebox}
\noindent
\begin{minipage}[t]{0.56\linewidth}
\methodtitle{EFlow (Ours)}
\end{minipage}\hfill
\begin{minipage}[t]{0.40\linewidth}
\raggedleft \metrics{54.92}{4.04}
\end{minipage}

\vspace{4pt}
\noindent
[CLS] the side. [SEP] i'm guessing that it is are based on a and point. [SEP] he called the old life,'bo the, of andal'to mr. [SEP] and his father. [SEP] the hon in given patients to an does not or a test. [SEP] i was not to for her. [SEP] by way, it is not to be a, but available to be funny as well - - just as a party. [SEP] i am waiting for it to happen for each firm. [SEP] craig ward made 30 saves for the defending st louisny and antoine vermo on its 13out, including six - [SEP]
\end{samplebox}

\vspace{4pt}

\nfeheader{steps = 264}

\vspace{4pt}

\begin{samplebox}
\noindent
\begin{minipage}[t]{0.56\linewidth}
\methodtitle{EFlow (Ours)}
\end{minipage}\hfill
\begin{minipage}[t]{0.40\linewidth}
\raggedleft \metrics{41.29}{4.09}
\end{minipage}

\vspace{4pt}
\noindent
[CLS] spent 15 years in prison at a \$ 1 million military camp in afghanistan. [SEP] there's a face on it, we can't look like our else. [SEP] he never lost the great of the before. [SEP] court's bemourned her on tuesday. [SEP] two teenagers have died from a serious wound in hospital. [SEP] the military alone was not enough to force the points. [SEP] adjusted for inflation, m. a. n. m. was up 4. 2 percent, but is up slightly. [SEP] " it is a job we have agreed for as well as with the new - community policing group. [SEP] [SEP]
\end{samplebox}

\end{minipage}
\hfill
\begin{minipage}[t]{0.485\linewidth}

\nfeheader{steps = 128}

\vspace{4pt}

\begin{samplebox}
\noindent
\begin{minipage}[t]{0.56\linewidth}
\methodtitle{EFlow}
\end{minipage}\hfill
\begin{minipage}[t]{0.40\linewidth}
\raggedleft \metrics{54.19}{4.21}
\end{minipage}

\vspace{4pt}
\noindent
[CLS] is scheduled to be monday. [SEP] they are appealing against two key yesterday. [SEP] i'm never sure how much i mean. [SEP] he said, " a friend has come a little down and has been making the most into again. " [SEP] by the time, last week, some 5, 000 tamil there were living in the, which on have year, an even the s at found. [SEP] the price of gasoline was 2. 7 \% compared with the previous year. [SEP] the next morning mr cameron said that he would not be able to give london to for the expense of defence. [SEP] they study the verys so that [SEP]
\end{samplebox}

\vspace{4pt}

\nfeheader{steps = 512}

\vspace{4pt}

\begin{samplebox}
\noindent
\begin{minipage}[t]{0.56\linewidth}
\methodtitle{EFlow (Ours)}
\end{minipage}\hfill
\begin{minipage}[t]{0.40\linewidth}
\raggedleft \metrics{49.89}{4.19}
\end{minipage}

\vspace{4pt}
\noindent
[CLS] - hitter, will win against the wall. [SEP] no wonder it was right? [SEP] the wall street agency said the biggest threat to the financial system is from individuals and the, caused of over \$ 3 million in assets. [SEP] the public library at the university is having a lowdown on film. [SEP] japan's olympic committee will meet again on monday. [SEP] it is a difficult ( and sometimes expensive ) decision to make with us, but it also takes time and effort to do with the usual for ofs. [SEP] of the world's three countries with low - on economies face up to 22pc of gdp, most has [SEP]
\end{samplebox}
\end{minipage}

\end{figurepanel}
\caption{Example LM1B generations from the EFlow teacher at $64$, $128$, $264$, and $512$ sampling steps. Generative perplexity (GPT-2-Large) and entropy are reported for the individual sample shown.}
\label{fig:eflow_lm1b_samples}
\end{figure}

\begin{figure}[t]
\centering
\begin{figurepanel}

\noindent
\begin{minipage}[t]{0.485\linewidth}
\nfeheader{steps = 1}

\vspace{8pt}

\begin{samplebox}
\noindent
\begin{minipage}[t]{0.56\linewidth}
\methodtitle{EFM (Ours)}
\end{minipage}\hfill
\begin{minipage}[t]{0.40\linewidth}
\raggedleft \metrics{47.8}{3.72}
\end{minipage}

\vspace{4pt}
\noindent
[CLS] the into being us. [CLS] it [CLS] for a the state. [CLS] more than a state. [CLS] after, we are. [CLS] the first time off the president and the who would - at least [CLS] : and after days - percent of, when would be their an, of, go to to off about in two about to to a about more than their second into [CLS]t they to more, when it to an into the deal. [CLS] from was the, [CLS].'over the are of you two. on make two being., who now have to to another,, one of the of its with from if alls [CLS]
\end{samplebox}

\vspace{8pt}

\nfeheader{steps = 4}

\vspace{8pt}

\begin{samplebox}
\noindent
\begin{minipage}[t]{0.56\linewidth}
\methodtitle{EFM (Ours)}
\end{minipage}\hfill
\begin{minipage}[t]{0.40\linewidth}
\raggedleft \metrics{55.2}{4.12}
\end{minipage}

\vspace{4pt}
\noindent
[CLS] second just to [CLS] his own city, while it's on to three - world countries such as the united states, too left for this year. [CLS] more than i do later has to me been asp. [CLS] i don't want three call he 24 to be [CLS], " she said. [CLS] he then can at least her once again again for two days. [CLS] but. i go at the university. [CLS] she was not [CLS] next week or there's already ever later [CLS]. [CLS], [CLS] 6 ( upi ) - - - they are another lead to the - john in an a through last - like [CLS]
\end{samplebox}

\end{minipage}
\hfill
\begin{minipage}[t]{0.485\linewidth}
\nfeheader{steps = 2}

\vspace{8pt}

\begin{samplebox}
\noindent
\begin{minipage}[t]{0.56\linewidth}
\methodtitle{EFM (Ours)}
\end{minipage}\hfill
\begin{minipage}[t]{0.40\linewidth}
\raggedleft \metrics{60.8}{4.01}
\end{minipage}

\vspace{4pt}
\noindent
[CLS] country. [CLS] but well it will not see after all in the high school of, tenn. [CLS] the most important of his first second in three years which has left over more on u. s. house. [CLS] it's do, and that's a good country in city high points. [CLS] well, one week, as well as he's trying to be up by half two points. [CLS] the 23 - year - old's house, they are more important is the country's own win, and when they are once in the game house because he is to work for the most next, but an [CLS]
\end{samplebox}

\vspace{8pt}

\nfeheader{steps = 8}

\vspace{8pt}

\begin{samplebox}
\noindent
\begin{minipage}[t]{0.56\linewidth}
\methodtitle{EFM (Ours)}
\end{minipage}\hfill
\begin{minipage}[t]{0.40\linewidth}
\raggedleft \metrics{44.2}{4.03}
\end{minipage}

\vspace{4pt}
\noindent
[CLS] every year. [CLS] he was one after his week's only, has been on top of the world. [CLS] he had his team on the close, at least in this week. [CLS] " i have people, who do not think it's close to keep up to be it in party, " the said. [CLS] a last week's. man of years last week, first four of the children - were made on tuesday with me. [CLS] that's not all. [CLS] but there have been four al - qaeda, which don't the right four two go to the country. they million ifs of these [CLS]
\end{samplebox}
\end{minipage}

\end{figurepanel}
\caption{Example LM1B generations from the EFM teacher at $1$, $2$, $4$, and $8$ sampling steps. Generative perplexity (GPT-2-Large) and entropy are reported for the individual sample shown.}
\label{fig:efm_lm1b_samples}
\end{figure}

\begin{figure}
    \centering
    \includegraphics[width=\linewidth]{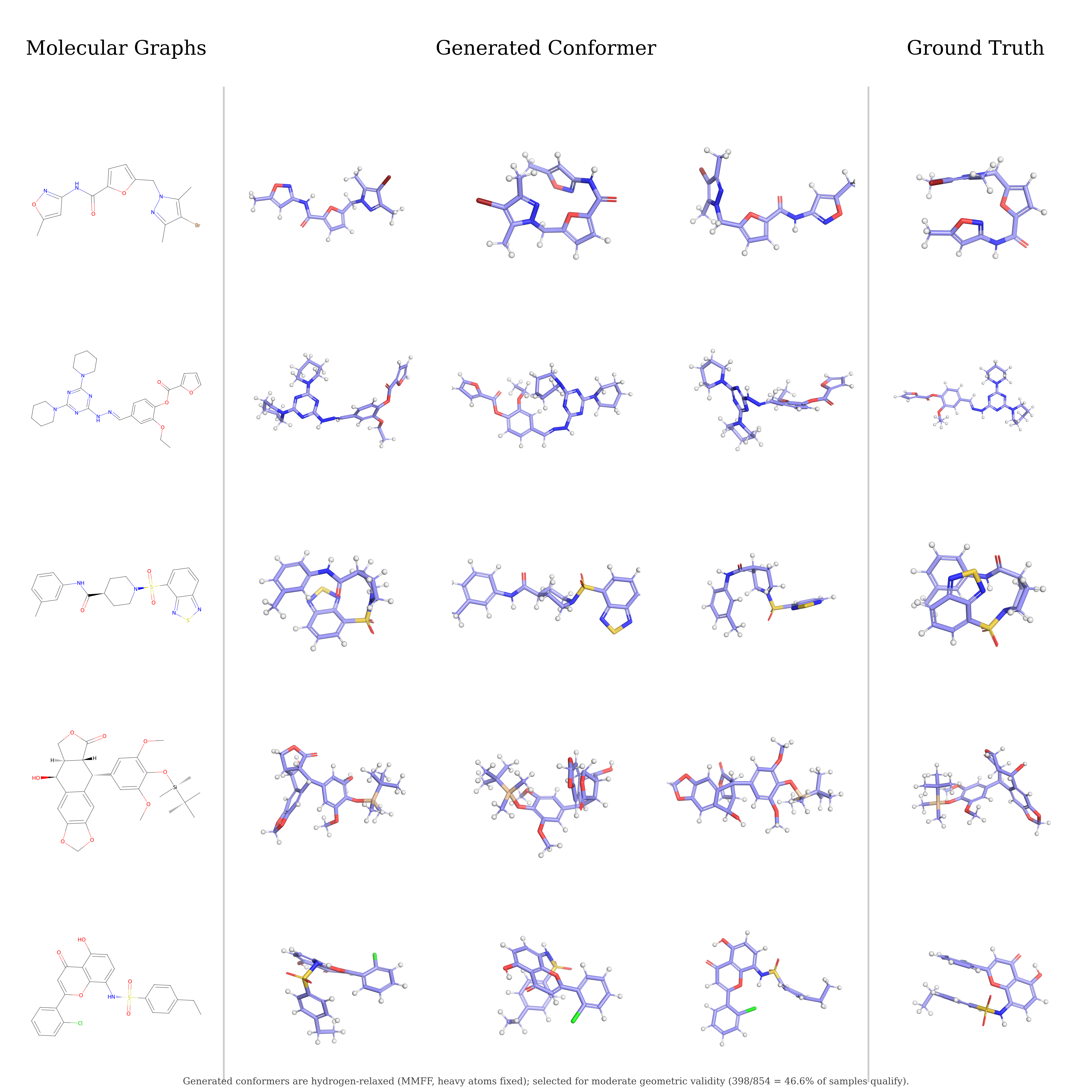}
    \caption{Molecular graph, three generated conformers with EFlow+R (30 steps), and ground-truth sample for 5 test molecules in GEOM-Drugs, visualized with PyMol \citep{delano2002pymol}.}
    \label{fig:eflow_drugs_grid}
\end{figure}

\begin{figure}
    \centering
    \includegraphics[width=\linewidth]{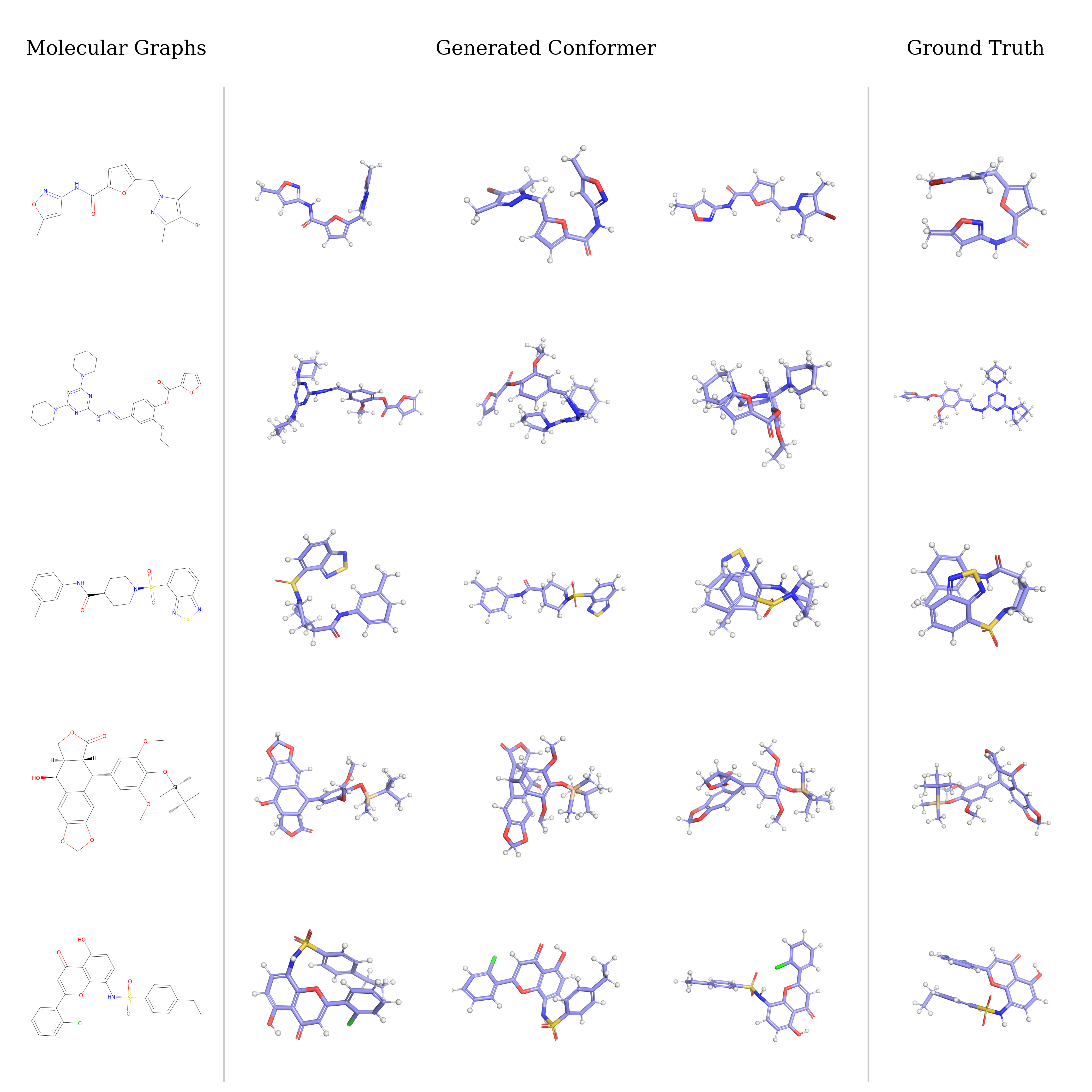}
    \caption{Molecular graph, three generated conformers with EFM+R (6 steps), and ground-truth sample for 5 test molecules in GEOM-Drugs, visualized with PyMol \citep{delano2002pymol}.}
    \label{fig:efm_drugs_grid}
\end{figure}

\end{document}